\documentclass{article}

\usepackage{booktabs}       
\usepackage{amsfonts}       
\usepackage{nicefrac}       
\usepackage{microtype}      
\usepackage{graphicx}
\usepackage{multirow}
\usepackage{amsmath}
\usepackage{amssymb}
\usepackage{subcaption}
\usepackage{paralist}
\usepackage{tabularx}
\usepackage{xcolor}
\usepackage{algpseudocode}
\usepackage{setspace}
\usepackage{enumitem}
\usepackage{subcaption}

\newcommand{\tens}[1]{%
  \mathbin{\mathop{\otimes}\limits_{#1}}%
}

\newcommand{\x}{\ensuremath{x}}

\newcommand{\z}{\ensuremath{z}}

\newcommand{\q}{\theta}
\newcommand{\f}{\phi}

\newcommand{\E}{\ensuremath{\mathbb{E}}}

\newcommand{\hide}[1]{}

\makeatletter
\DeclareRobustCommand{\cev}[1]{%
  \mathpalette\do@cev{#1}%
}
\newcommand{\do@cev}[2]{%
  \fix@cev{#1}{+}%
  \reflectbox{$\m@th#1\vec{\reflectbox{$\fix@cev{#1}{-}\m@th#1#2\fix@cev{#1}{+}$}}$}%
  \fix@cev{#1}{-}%
}
\newcommand{\fix@cev}[2]{%
  \ifx#1\displaystyle
    \mkern#23mu
  \else
    \ifx#1\textstyle
      \mkern#23mu
    \else
      \ifx#1\scriptstyle
        \mkern#22mu
      \else
        \mkern#22mu
      \fi
    \fi
  \fi
}

\definecolor[named]{Blue}{cmyk}{1,0.1,0,0.1}
\definecolor[named]{Yellow}{cmyk}{0,0.16,1,0}
\definecolor[named]{Orange}{cmyk}{0,0.42,1,0.01}
\definecolor[named]{Red}{cmyk}{0,0.90,0.86,0}
\definecolor[named]{LightBlue}{cmyk}{0.49,0.01,0,0}
\definecolor[named]{Green}{cmyk}{0.20,0,1,0.19}
\definecolor[named]{Purple}{cmyk}{0.55,1,0,0.15}
\definecolor[named]{DarkBlue}{cmyk}{1,0.58,0,0.21}

\usepackage[acronym,smallcaps,nowarn,section,nogroupskip,nonumberlist]{glossaries}
\glsdisablehyper{}
\newacronym{SCFM}{scfm}{stochastic control-flow model}
\newacronym{WS}{ws}{wake-sleep}
\newacronym{BWS}{bws}{basic wake-sleep}
\newacronym{RWS}{rws}{reweighted wake-sleep}
\newacronym{ELBO}{elbo}{evidence lower bound}
\newacronym{VAE}{vae}{variational autoencoder}
\newacronym{IWAE}{iwae}{importance weighted autoencoder}
\newacronym{KL}{kl}{Kullback-Leibler}
\newacronym{SGD}{sgd}{stochastic gradient descent}
\newacronym{VIMCO}{vimco}{variational inference for Monte Carlo objectives}
\newacronym{WW}{ww}{wake-wake}
\newacronym{WWS}{wws}{wake-wake-sleep}
\newacronym{AIR}{air}{Attend, Infer, Repeat}
\newacronym{ESS}{ess}{effective sample size}
\newacronym{REINFORCE}{reinforce}{Reinforce gradient estimator}
\newacronym{IS}{is}{importance sampling}
\newacronym{GMM}{gmm}{Gaussian mixture model}
\newacronym{MNIST}{mnist}{hand-written digit dataset}
\newacronym{RELAX}{relax}{RELAX gradient estimator}
\newacronym{REBAR}{rebar}{REBAR gradient estimator}
\newacronym{PMF}{pmf}{probability mass function}
\newacronym{MLP}{mlp}{multilayer perceptron}
\newacronym{RNN}{rnn}{recurrent neural network}
\newacronym{PCFG}{pcfg}{probabilistic context free grammar}
\newacronym{ADAM}{adam}{ADAM}
\glsunset{ADAM}

\usepackage{amsthm}
\newtheorem{proposition}{Proposition}
\theoremstyle{definition}
\newtheorem{definition}{Definition}

\newcommand{\given}{\lvert}
\newcommand{\pw}{\overset{\text{p.w.}}{\sim}
}

\usepackage{hyperref}

\usepackage[accepted]{icml2020}

\icmltitlerunning{Amortized Population Gibbs Samplers with Neural Sufficient Statistics}

\begin{document}

\twocolumn[
\icmltitle{Amortized Population Gibbs Samplers with Neural Sufficient Statistics}



\icmlsetsymbol{equal}{*}

\begin{icmlauthorlist}
\icmlauthor{Hao Wu}{neu}
\icmlauthor{Heiko Zimmermann}{neu}
\icmlauthor{Eli Sennesh}{neu}
\icmlauthor{Tuan Anh Le}{mit}
\icmlauthor{Jan-Willem van de Meent}{neu}
\end{icmlauthorlist}

\icmlaffiliation{neu}{Khoury College of Computer Sciences, Northeastern University, Boston, MA, USA}
\icmlaffiliation{mit}{Department of Brain and Cognitive Sciences, Massachusetts Institute of Technology, Cambridge, MA, USA}

\icmlcorrespondingauthor{Hao Wu}{haowu@ccs.neu.edu}

\icmlkeywords{Machine Learning, ICML}

\vskip 0.3in
]



\printAffiliationsAndNotice{}  

\begin{abstract}

We develop amortized population Gibbs (APG) samplers, a class of scalable methods that frames structured variational inference as adaptive importance sampling. APG samplers construct high-dimensional proposals by iterating over updates to lower-dimensional blocks of variables. We train each conditional proposal by minimizing the inclusive KL divergence with respect to the conditional posterior. To appropriately account for the size of the input data, we develop a new parameterization in terms of neural sufficient statistics. Experiments show that APG samplers can train highly structured deep generative models in an unsupervised manner, and achieve substantial improvements in inference accuracy relative to standard autoencoding variational methods.

\end{abstract}

\section{Introduction}
\label{introduction}
Many inference tasks involve hierarchical structure. For example, when modeling a corpus of videos that contain moving objects, we may wish to reason about three levels of representation. At the corpus level, we can reason about the distribution over objects and dynamics of motion. At the instance level, i.e.~for each video, we can reason the visual features and trajectories of individual objects. Finally, at a data point level, i.e.~for each video frame, we can reason about the visibility of objects and their position.



In the absence of supervision, uncovering hierarchical structure from data requires inductive biases. Deep generative models let us incorporate biases in the form of priors that mirror the structure of the problem domain. 
By parameterizing conditional distributions with neural networks, we hope to learn models that use corpus-level characteristics (e.g.~the appearance of objects and the motion dynamics) to make data-efficient inferences about instance-level variables (e.g.~object appearance and positions).

In recent years we have seen applications of amortized variational methods to inference in hierarchical domains, such as object detection \cite{eslami2016attend}, modeling of user reviews \cite{esmaeili2019structured}, and object tracking \cite{kosiorek2018sequential}. These approaches build on the framework of variational autoencoders (VAEs) \cite{kingma2013auto-encoding, rezende2014stochastic} to train structured deep generative models. However, scaling up these approaches to more complex domains still poses significant challenges. Whereas it is easy to train VAEs on large corpora of data, it is not easy to train structured encoders for models with a large number of (correlated) variables at the instance level. 

In this paper, we develop methods for amortized inference that scale to structured models with hundreds of latent variables. Our approach takes inspiration from work by \citet{johnson2016composing}, which develops methods for efficient inference in deep generative models with conjugate-exponential family priors. In this class of models, we can perform inference using variational Bayesian expectation maximization (VBEM)  \cite{beal2003variational,bishop2006pattern,wainwright2008graphical}, which iterates between closed-form updates to blocks of variables. VBEM is computationally efficient, often converges in a small number of iterations, and scales to a large number of variables. However, VBEM is also model-specific, difficult to implement, and only applicable to a restricted class of conjugate-exponential models.

To overcome these limitations, we develop a more general approach that frames variational inference as importance sampling with amortized proposals. We propose \emph{amortized population Gibbs} samplers, a class of methods that iterate between conditional proposals to blocks of variables. To train these proposals, we minimize the inclusive KL divergence w.r.t.~the conditional posterior. We use proposals in a sequential Monte Carlo (SMC) sampler to reduce the variance of importance weights, which improves efficiency during training and inference at test time.

Our experiments establish that APG samplers can capture the hierarchical structure in deep generative models in an unsupervised manner.
In Gaussian mixture models, where Gibbs updates can be computed in closed form, the learned proposals converge to the conditional posteriors.
To evaluate the performance on non-conjugate models, we design a deep generative mixture model and an unsupervised tracking model. We compare APG samplers to reweighted wake-sleep methods, a bootstrapped population sampler, and a HMC-augmented baseline.
In all experiments APG samplers can capture the hierarchical characteristics of the problem and outperformed existing methods in terms of training efficiency and inference accuracy.

We summarize the contributions of this paper as follows: 
\begin{enumerate}[labelwidth=0.5em,labelsep=0.5em,leftmargin=1.0em,topsep=0em,itemsep=\parsep]
    \vspace{-0.3em}
    \item We develop APG samplers, a new class of amortized inference methods that employ approximate Gibbs conditionals to iteratively improve samples.
    \vspace{-0.3em}
    \item We present a novel parameterization using neural sufficient statistics to learn proposals that generalize across input instances that vary in size.
    \vspace{-0.3em}
    \item We show that APG samplers can be used to train structured deep generative models with 100s of instance-level variables in an unsupervised manner.
    \vspace{-0.3em}
    \item We demonstrate substantial improvements in terms of test-time inference efficiency relative to baseline methods.
\end{enumerate}



\section{Background}
\label{sec:background}

In probabilistic machine learning, we commonly define a generative model in terms of a joint distribution $p_\q(\x, \z)$ over data $\x$ and latent variables $\z$. Given a model specification, we are interested in reasoning about the posterior distribution $p_\q(\z \mid \x)$ conditioned on data $\x$ sampled from an (unknown) data distribution $p^\textsc{data}(\x)$, which we in practice approximate using an empirical distribution over training data $\hat{p}(x)$. Amortized variational inference methods approximate the posterior $p_\q(\z \mid \x)$ with a distribution $q_\f(\z \mid \x)$ in some tractable variational family.

A widely used class of deep probabilistic models are variational autoencoders \cite{kingma2013auto-encoding, rezende2014stochastic}.
These models are trained by maximizing the stochastic lower bound (ELBO) in Equation~\ref{eq:elbo} on the log marginal likelihood, which is equivalent to minimizing the exclusive KL divergence $\mathrm{KL}(q_\f(z | x) || p_\q(z | x))$. 
\begin{align}
    \label{eq:elbo}
    \mathcal{L} (\f)
    &
    = 
    \E_{\hat{p}(x) \: q_\f(\z \mid \x)}
    \left[
       \log \frac{p_\q(\x, \z)}{q_\f(\z \mid \x)}
    \right] 
\end{align}
When the distribution specified by the encoder is reparameterizable, we can compute Monte Carlo estimates of the gradient of this objective using pathwise derivatives. 
Non-reparameterizable cases, such as models with discrete variables, require likelihood-ratio estimators  \cite{williams1992simple} 
which can have a high variance. A range of approaches have been proposed to reduce this variance, including reparameterized continuous relaxations \cite{maddison2017concrete,jang2017categorical}, credit assignment techniques \cite{weber2019credit}, and other control variates \cite{mnih2016variational,tucker2017rebar,grathwohl2018backpropagation}. 

Reweighted wake-sleep (RWS) methods \cite{bornschein2014reweighted} sidestep the need for reparameterization by minimizing the stochastic upper bound in Equation~\ref{eq:eubo}, which is equivalent to minimizing an inclusive KL divergence:
\begin{align}
    \label{eq:eubo}
    \mathcal{U} (\f)
    &
    = 
    \E_{\hat{p}(x) \, p_\q(\z \mid \x)}
    \left[
       \log \frac{p_\q(\x, \z)}{q_\f(\z \mid \x)}
    \right]
    .
\end{align}
Here we can compute a self-normalized gradient estimator using samples $z \sim q_\f(z \mid x)$ and weights $w = \frac{p_\q(x, z)}{q_\f(z \mid x)}$
\begin{align}
\label{eq:grad-phi-rws}
    - \nabla_\f
    \,
    \mathcal{U} (\f)
    \simeq
    \sum_{l=1}^L
    \frac{w^l}{\sum_{l'} w^{l'}}
    \nabla_\f
    \log q_\f(z^l \,|\, x)
    .
\end{align}
This gradient estimator has a number of advantages over the estimator used in standard VAEs \cite{le2019revisiting}. Notably, it only requires that the \emph{proposal density} is differentiable, whereas reparameterized estimators require that the \emph{sample} itself is differentiable. This is a milder condition that holds for most distributions of interest, including those over discrete variables. Moreover, minimizing the inclusive KL reduces our risk of learning a proposal that collapses to a single mode of a multi-modal posterior \cite{le2019revisiting}. 

\section{Amortized Population Gibbs Samplers}
\label{sec:amortized-gibbs}
To generate high-quality samples from the posterior in an incremental manner, we develop methods that are inspired by EM and Gibbs sampling, both of which perform iterative updates to blocks of variables. Concretely, we assume that the latent variables in the generative model decompose into blocks $\z = \{\z_1, \ldots, \z_B\}$ and train proposals $q_\f(z_b \mid x, z_{-b})$ that update the variables in a each block $z_{b}$ conditioned on variables in the remaining blocks $\z_{-b} = z \setminus \{z_b\}$.

Starting from a sample $q_\f(\z^{1} \mid \x)$ from a standard encoder, we  generate a sequence of samples $\{z^1, \ldots, z^K\}$ by performing a sweep of conditional updates to each block $z_b$
\begin{align}
    \label{eq:approx-gibbs-kernel}
    q_\f \big( \z^k \mid \x, \z^{k-1} \big)
    &=
    \prod_{b=1}^B
    q_\f(\z^k_b \mid \x, \z^{k}_{\prec b}, \z^{k-1}_{\succ b})
    ,
\end{align}
where $\z_{\prec b} = \{z_i \mid i < b\}$ and $\z_{\succ b} = \{z_i \mid i > b\}$. Repeatedly applying sweep updates then yields a proposal
\begin{equation*}
    q_\f(z^1, \ldots, z^K \mid x) 
    =
    q_\f(z^1 \mid x)
    \prod_{k=2}^K
    q_\f(z^k \mid \x, z^{k-1}).
\end{equation*}
We want to train proposals that improve the quality of each sample $z^k$ relative to that of the preceding sample $z^{k-1}$. There are two possible strategies for accomplishing this. One is to define an objective that minimizes the discrepancy between the marginal $q_\f(z^K \mid x)$ for the final sample and the posterior $p_\q(z^K \mid x)$. This corresponds to learning a sweep update $q_\f(z^k \mid x, z^{k-1})$ that transforms the initial proposal to the posterior in exactly $K$ sweeps. An example of this type of approach, albeit one that does not employ block updates, is the recent work on annealing variational objectives \cite{huang2018improving}.

In this paper, we will pursue a different approach. Instead of transforming the initial proposal in exactly $K$ steps, we learn a sweep update that leaves the target density \emph{invariant}
\begin{equation*}
    p_\q(\z^k \,|\, \x) = \int \: p_\q(\z^{k-1} \,|\, \x) \: q_\f(\z^k \,|\,\x, \z^{k-1}) d z^{k-1}.
\end{equation*}
When this condition is met, the proposal $q_\f(z^1, \ldots z^K \,|\, x)$ is a Markov Chain whose stationary distribution is the posterior. 
This means a sweep update learned at training time can be applied at test time to iteratively improve sample quality, without requiring a pre-specified number of updates $K$.

When we require that each block update $q_\f(z'_b \mid x, z_{-b})$ leaves the target density invariant, 
\begin{align}
    \label{eq:gibbs-invariance}
    p_\q(z'_b, z_{-b} \mid x)
    &= 
    \int 
    p_\q(z_b, z_{-b} \mid x) \:
    q_\f(z'_b, \mid x, z_{-b}) 
    dz_{b}
    \nonumber
    \\
    &= 
    p_\q(z_{-b} \mid x) \:
    q_\f(z'_b \mid x, z_{-b}),
\end{align}
then each update must equal the exact conditional posterior 
$q_\f(z'_b \mid x, z_{-b}) = p_\q(z'_b \mid x, z_{-b})$. In other words, when the condition in Equation~\ref{eq:gibbs-invariance} is met, the proposal $q_\f(z^1, \ldots z^K \,|\, x)$ is a Gibbs sampler.

\subsection{Variational Objective} 
To learn each of the block proposals $q_\f(z_b \mid x, z_{-b})$ we will minimize the inclusive KL divergence $\mathcal{K}_b(\f)$
\begin{align*}
    \E_{\hat{p}(x)p_\q(z_{-b} | x)}
    \biggl[
        \text{KL}\left(
            p_\q(z_{b} \mid x, z_{-b})
            \,||\,
            q_\f(z_{b} \mid x, z_{-b})
        \right)
    \biggr]
\end{align*}
Unfortunately, this objective is intractable, since we are not able to evaluate the density of the true marginal $p_\q(z_{-b} \mid x)$, nor that of the conditional $p_\q(z_{b} \mid z_{-b}, x)$. 
As we will discuss in Section~\ref{sec:experiments}, this has implications for the evaluation of the learned proposals, since we cannot compute a lower or upper bound on the log marginal likelihood as in other variational methods. However, it is nonetheless possible to approximate the gradient of the objective 
\begin{align*}
    -\nabla_\f \mathcal{K}_b(\f)
    &=
    \E_{
    \hat{p}(x) \:
    p_\q(z_{b}, z_{-b} | x)
    }
    \left[
    \nabla_\f
    \log q_\f(z_{b} \,|\, x, z_{-b})
    \right].
\end{align*}
We can estimate this gradient using any Monte Carlo method that generates samples from the posterior $z \sim p_\q(z \,\mid\, x)$. In the next section, we will use the learned proposals to define an importance sampler, which we will use to obtain a set of weighted samples $\{(w^l, \, z^l)\}_{l=1}^L$ to compute a self-normalized gradient estimate
\begin{align}
    \label{eq:grad-self-normalized}
    -\nabla_\f \mathcal{K}_b(\f)
    \simeq
    \sum_{l=1}^L
    \frac{w^l}{\sum_{l'} w^{l'}}
    \nabla_\f
    \log q_\f(z^l_{b} \,|\, x, z^l_{-b}).
\end{align}
When we also want to learn a generative model $p_\q(\x, \z)$, we can compute a similar self-normalized gradient estimate 
\begin{align}
    \nabla_\q \log p_\q(\x) 
    &=
    \mathbb{E}_{p_\q(\z | \x)} 
    \left[
    \nabla_\q \log p_\q(\x, \z)
    \right],
    \label{eq:grad-theta}
    \\
    &
    \simeq
    \sum_{l=1}^L
    \frac{w^l}{\sum_{l'} w^{l'}}
    \nabla_\q
    \log p_\q(x, z^l).
    \nonumber
\end{align}
to maximize the marginal likelihood.
The identity in Equation~\ref{eq:grad-theta} holds due to the identity \mbox{$\E_{p_\q(z|x)}[\nabla_\q \log p_\q(z|x)]=0$} (see Appendix \ref{appendix:grad-theta} for details). 

In the following, we will describe an adaptive importance sampling scheme which generates high quality samples to accurately estimate gradients for the conditional proposal $q_\f(z_b | x, z_{-b})$ and an initial one-shot encoder $q_\f(z | x)$.

\subsection{Generating High Quality Samples}


A well-known limitation of self-normalized importance samplers is that weights can have a high variance when latent variables are high-dimensional and/or correlated, which is common in models with hierarchical structure. Moreover, this variance gives rise to bias in self-normalized estimators. This bias arises from the approximation of the normalizing factor $1/p_\theta(x) \simeq 1/\hat{Z}$ using the average weight $\hat{Z} = \frac{1}{L} \sum_{l} w^l$. The average weight is an unbiased estimate of the marginal likelihood $\mathbb{E}[\hat{Z}] = p_\theta(x)$. However $1/\hat{Z}$ is not an unbiased estimate of the normalizing factor, since by Jensen's inequality $\mathbb{E}[1/\hat{Z}] \ge 1/p_\theta(x)$, resulting in a bias that is directly tied to the variance of the weights.

In practice, limitations on GPU memory and computation time imply that the estimators in Equations~\ref{eq:grad-self-normalized} and~\ref{eq:grad-theta} need to employ a small budget $L$. This means that the bias will likely be high, particularly during the early stages of training when the proposal poorly approximates the posterior. However, even a biased estimator can serve to improve the quality of a proposal in practice. A high-bias estimator will have a low effective sample size, which is to say that it will assign a normalized weight close to $1$ to the best proposal in the set and weight $0$ to all other samples. The gradient in Equation~\ref{eq:grad-self-normalized} will therefore initially serve to increase the probability of the highest-weight sample. As long as doing so improves the quality of the proposal, RWS-style optimization will gradually reduce the bias as training proceeds.

To help mitigate the bias of self-normalized estimators, we use a sequential Monte Carlo sampler \citep{delmoral2006sequential} to reduce weight variance. SMC methods \cite{doucet2001sequential} decompose a high-dimensional sampling problem into a sequence of lower-dimensional problems. They do so by combining two ideas: (1) sequential importance sampling, which defines a proposal for a sequence of variables using a sequence of conditional proposals, (2) resampling, which selects intermediate proposals with probability proportional to their weights to improve sample quality. 

SMC methods are most commonly used in state space models to generate proposals for a sequence of variables by proposing one variable at a time. We here consider SMC \emph{samplers} (see  Algorithm~\ref{alg:smcs}), a subclass of SMC methods that interleave resampling with the application of a transition kernel. The distinction between SMC methods for state space models and SMC samplers is subtle but important. Whereas the former generate proposals for a sequence of variables $z_{1:t}$ by proposing $z_t \sim q(z_t \mid z_{1:t-1})$ to \emph{extend} the sample space at each iteration, SMC samplers can be understood as an importance sampling analogue to Markov chain Monte Carlo (MCMC) methods \cite{brooks2011handbook}, which construct a Markov chain $z^{1:K}$ by generating a proposal $z^k$ from a transition kernel $q(z^k \mid z^{k-1})$ at each iteration.



\begin{algorithm}[!t]
\setstretch{1.2}
  \caption{Sequential Monte Carlo sampler}
  \label{alg:smcs}
\begin{algorithmic}[1]
    \small
    \State \textbf{Input} Model $\gamma^{1:K}$, Proposals $q^{1:K}$, Number of particles $L$
    \For{$l = 1$ \textbf{to} $L$}
        \State $z^{1,l} \sim q^1(\cdot)$\Comment{Propose}
        \State $w^{1,l} = \frac{\gamma^1(z^{1,l})}{q^1(z^{1,l})}$\Comment{Weigh}
    \EndFor
    \For{$k = 2$ \textbf{to} $K$}
      \State$z^{k-1,1:L}, w^{k-1,1:L}=\textsc{resample}(z^{k-1,1:L}, w^{k-1,1:L})$

      \For{$l = 1$ \textbf{to} $L$}
          \State $z^{k,l} \sim q^k(\cdot \mid z^{k-1,l})$\Comment{Propose}\label{line:apg-propose}
          \State $w^{k,l} = \frac{\gamma^k(z^{k,l}) r^{k-1}(z^{k-1,l} \mid z^{k,l})}{\gamma^{k-1}(z^{k-1,l}) q^k(z^{k,l} \mid z^{k-1,l})}w^{k-1,l}$\Comment{Weigh}
      \EndFor
    \EndFor
    \State \textbf{Output} Weighted Samples $\{z^{K, l}, w^{K, l}\}_{l=1}^L$
\end{algorithmic}
\end{algorithm}

To understand how approximate Gibbs proposals can be incorporated into a SMC sampler, we will first define a sequential importance sampler (SIS), which decomposes the importance weight into a sequence of \emph{incremental} weights. SIS considers a sequence of unnormalized target densities $\gamma^1(z^1), \gamma^2(z^{1:2}), \dots, \gamma^K(z^{1:K})$. If we now consider an initial proposal $q^1(z^1)$, along with a sequence of conditional proposals $q^k(z^k \mid z^{1:k-1})$, then we can recursively construct a sequence of weights $w^k = v^k w^{k-1}$ by assuming $w^1 = \gamma^1(z^1) / q^1(z^1)$ and defining the incremental weight
\begin{align*}
    v^k 
    &=
    \frac{\gamma^k(z^{1:k})}{\gamma^{k-1}(z^{1:k-1}) q^k(z^k \mid z^{1:k-1})}.
\end{align*}
This construction ensures that, at step $k$ in the sequence, we have a weight $w^k$ relative to the intermediate  density $\gamma^k(z^{1:k})$ of the form (see Appendix~\ref{appendix:sis-weight})
\begin{align*}
    w^k
    = 
    \frac{\gamma^k(z^{1:k})}
         {q^1(z^1) \prod_{k'=2}^k q^{k'}(z^{k'} \mid z^{1:k'-1})}.
\end{align*}

We will now consider a specific sequence of intermediate densities that are defined using a \emph{reverse kernel} $r(z' \mid z)$:
\begin{align*}
    \gamma^k(z^{1:k})
    =
    p_\q(x,z^k) \prod_{k'=2}^{k} r(z^{k'-1} \mid z^{k'})
    .
\end{align*}
This defines a density on an \emph{extended space} which admits the marginal density
\begin{align*}
    \gamma^k(z^k) = \int \gamma^k(z^{1:k}) \: dz^{1:k-1} = p_\q(x, z^k).
\end{align*}
This means that at each step $k$, we can treat the preceding samples $z^{1:k-1}$ as \emph{auxiliary variables}; if we generate a proposal $z^{1:k}$ and simply disregard $z^{1:k-1}$, then the pair $(w^k, z^k)$ is a valid importance sample relative to $p_\q(z^k \mid x)$. If we additionally condition proposals on $x$, the incremental weight for this particular choice of target densities is
\begin{align}
    \label{eq:incremental-weight-forward-reverse}
    v^k 
    = 
    \frac{p_\q(x,z^k) \: r(z^{k-1} \mid x, z^k)}
         {p_\q(x,z^{k-1}) \: q(z^{k} \mid x, z^{k-1})}
    .
\end{align}
This construction is valid for any choice of \emph{forward kernel}  $q(z^k \mid z^{k-1})$ and reverse kernel $r(z^{k-1} \mid z^{k})$, but the weight variance will depend on the choice of kernel. Given a forward kernel, the \emph{optimal} reverse kernel is
\begin{align*}
    r(z^{k-1} \mid x, z^k)
    =   
    \frac{p_\q(x, z^{k-1})}
         {p_\q(x, z^k)}
    q(z^{k} \mid x, z^{k-1})
    .
\end{align*}
For this choice of kernel, the incremental weights are 1, which minimizes the variance of the weights $w^k$.

\begin{algorithm}[!tb]
  \caption{Amortized Population Gibbs Sampler}
  \label{alg:amortized-gibbs}
\begin{algorithmic}[1]
\small
\State \textbf{Input} Data x, Model $\small{p_\q (x, z)}$ \\
  \hspace{2.4em} Proposals $\small{q_\f (z | x), \{q_\f(z_b | x, z_{-b})\}_{b=1}^B}$
  \State $g_\phi = 0, g_\theta = 0$\Comment{Initialize gradient to 0}\label{line:init-grad}
  \For{$l = 1$ \textbf{to} $L$}\label{line:rws-loop}\Comment{Initial proposal}
      \State $z^{1,l} \sim q_\phi(\cdot \mid x)$\Comment{Propose}\label{line:rws-propose}
      \State $w^{1,l} = \frac{p_\theta(x, z^{1,l})}{q_\phi(z^{1,l} \mid x)}$\Comment{Weigh}\label{line:rws-weight}
  \EndFor
  \State $g_\phi = g_\phi + \sum_{l = 1}^L \frac{w^{1,l}}{\sum_{l' = 1}^L w^{1,l'}} \nabla_\f \log q_\phi(z^{1,l} \mid x)$\label{line:rws-grad-phi}
  \State $g_\theta = g_\theta + \sum_{l = 1}^L \frac{w^{1,l}}{\sum_{l' = 1}^L w^{1,l'}} \nabla_\q \log p_\theta(x, z^{1,l})$\label{line:rws-grad-theta}
  \For {$k = 2$ \textbf{to} $K$}\label{line:sweep-loop}\Comment{Gibbs sweeps}
    \State $\tilde{z}^{1:L}, \tilde{w}^{1:L} = z^{k-1,1:L}, w^{k-1,1:L}$ \label{line:apg-sweep-begin}
    \For{$b = 1$ \textbf{to} $B$}\label{line:block-loop}\Comment{Block updates}
      \State $\tilde{z}^{1:L}, \tilde{w}^{1:L}=\textsc{resample}(\tilde{z}^{1:L}, \tilde{w}^{1:L})$\label{line:resample} 
        \For{$l = 1$ \textbf{to} $L$}\label{line:apg-sample-loop}
          \State $\tilde{z}'^{\:l}_b \sim q_\phi(\cdot \mid x, \tilde{z}_{-b}^l)$\Comment{Propose}\label{line:apg-propose}
          \State \label{line:apg-weight} $\tilde{w}^l = \frac{p_\theta(x, \tilde{z}'^{\:l}_{b}, \tilde{z}^{\:l}_{-b}) q_\f(\tilde{z}^{\:l}_b \mid x, \tilde{z}^{\:l}_{-b})}{p_\theta(x, \tilde{z}^{\:l}_b, \tilde{z}^{\:l}_{-b}) q_\phi(\tilde{z}'^{\:l}_b \mid x, \tilde{z}^{\:l}_{-b})}\tilde{w}^l$ \Comment{Weigh}
          \State \label{line:apg-reassign}$\tilde{z}^{\:l}_b = \tilde{z}'^{\:l}_b$ \Comment{Reassign}
      \EndFor
      \State $g_\phi=g_\phi+ \sum_{l = 1}^L\frac{\tilde{w}^l}{\sum_{l'= 1}^L \tilde{w}^{l'}} \nabla_\f\log q_\phi(\tilde{z}^{\:l}_b \mid x, \tilde{z}^{\:l}_{-b})$\label{line:apg-grad-phi}
      \State $g_\theta = g_\theta + \sum_{l = 1}^L \frac{\tilde{w}^l}{\sum_{l' = 1}^L \tilde{w}^{l'}} \nabla_\q \log p_\theta(x, \tilde{z}^l)$ \label{line:apg-grad-theta}
     \EndFor
     \State $z^{k,\:1:L}, w^{k,\:1:L} = \tilde{z}^{1:L}, \tilde{w}^{1:L}$\label{line:apg-sweep-end}
     \vspace{0.5em}
  \EndFor
  \State \textbf{Output} Gradient $g_\phi$, $g_\theta$, Weighted Samples $\{z^{K,l},w^{K,l}\}_{l=1}^L$
\end{algorithmic}
\end{algorithm}

We here propose to use the approximate Gibbs kernel $q_\f(z^k |\, x, z^{k-1})$ as both the forward and the reverse kernel. When the approximate Gibbs kernel converges to the actual Gibbs kernel, this choice becomes optimal, since in limit the kernel satisfies detailed balance 
\begin{align*}
    p_\q(x, z^k) 
    \,
    q_\f(z^{k-1} \,|\, x, z^{k})
    =
    p_\q(x, z^{k-1}) 
    \,
    q_\f(z^{k} \,|\, x, z^{k-1}).
\end{align*}
In general, the weights $w^k$ in SIS will have a high variance when $z$ is high-dimensional or correlated. Moreover, we sample these variables $K$ times. When the approximate Gibbs kernel converges to the true kernel, this will not increase the variance of weights (since $v^k=1$ in limit). But during training variance of weights $w^k$ will increase with $k$, since we are jointly sampling an entire Markov chain.



To overcome this problem, SMC samplers interleave applications of the transition kernel with a \emph{resampling} steps (see Appendix~\ref{appendix:resample-algo} for Algorithm), which generates a new set of equally weighted samples by selecting samples from the current sample set with probability proportional to their weight.
APG samplers employ resampling after each block update $z_b \sim q_\phi(z_b | x, z_{-b})$, which results in equal incoming weights in the subsequent block update. This allows us to compute the gradient estimate of  block proposals based on the contribution of the incremental weight 
\begin{align*}
    v
    = 
    \frac{p_\q(x, z'_b, z_{-b}) \: q_\f(z_b \mid  x, z_{-b})}
         {p_\q(x, z_b, z_{-b}) \: q_\f(z'_b \mid  x, z_{-b})}
    .
\end{align*}
These incremental weights will generally have a much lower variance than full weights. This can improve sample-efficiency for both gradient estimation and inference, particularly in models with high-dimensional correlated variables.

We refer to this implementation of a SMC sampler as an \emph{amortized population Gibbs} sampler, and summarize all the steps of the computation in Algorithm~\ref{alg:amortized-gibbs}. In Appendix~\ref{appendix:proof-algo}, we prove that this algorithm is correct using an argument based on proper weighting. 

\section{Neural Sufficient Statistics}
\label{sec:neural-sufficient-statistics}

Gibbs sampling strategies that sample from exact conditionals rely on conjugacy relationships. We assume a prior and likelihood that can both be expressed as exponential families
\begin{align*}
    p(\x \mid \z) 
    &= 
    h(\x) \exp \{ 
        \eta(\z)^\top \: T(\x)  
        -\log A(\eta(\z)) \}, 
    \\
    p(\z ) 
    &= 
    h(\z) \exp \{ 
        \lambda^\top T(\z) 
        - \log A(\lambda) \}.
\end{align*}
where $h(\cdot)$ is a base measure, $T(\cdot)$ is a vector of sufficient statistics, and $A(\cdot)$ is a log normalizer. The two densities are jointly conjugate when $T(\z) = (\eta(\z), -\log A(\eta(\z)))$.
In this case, the posterior distribution lies in the same exponential family as the prior.
Typically, the prior $p(\z \mid \lambda)$ and likelihood $p(x \mid z)$ are not jointly conjugate, but it is possible to identify conjugacy relationships at the level of individual blocks of variables, 
\begin{align*}
    p(\z_b \mid \z_{-b}, \x)
    \propto
    h(z_b) 
    \exp \big\{ 
        &
        (\lambda_{b,1} + T(\x, \z_{-b}))^\top T(\z_b) 
        \\
        &
        -
        (\lambda_{b,2} \!+\! 1) 
        \log A(\eta(\z_b))
    \big\}
    .
\end{align*}
In general, these conjugacy relationships will not hold. However, we can still take inspiration to design variational distributions that make use of conditional independencies in a model. We assume that each of the approximate Gibbs updates $q_\f(\z_b \mid x, \z_{-b})$ is an exponential family, whose parameters are computed from a vector of prior parameters $\lambda$ and a vector of neural sufficient statistics $T_\f(\x, \z_{-b})$
\begin{align*}
    q_\f(\z_b \mid x, z_{-b}) 
    = 
    p(\z_b \mid \lambda + T_\f(x, z_{-b}))
    .
\end{align*}
This parameterization has a number of desirable properties. Exponential families are the largest-entropy distributions that match the moments defined by the sufficient statistics (see e.g.~\citet{wainwright2008graphical}), which is helpful when minimizing the inclusive KL divergence. In exponential families it is also more straightforward to control the entropy of the variational distribution. In particular, we can initialize $T_\f(x, z_{-b})$ to output values close to zero in order to ensure that we initially propose from a prior and/or regularize $T_\f(x, z_{-b})$ to help avoid local optima.

A useful case arises when the data $\x = \{x_1, \ldots, x_N\}$ are independent conditioned on $\z$. In this setting, it is often possible to partition the latent variables $\z = \{\z^\textsc{g}, \z^\textsc{l}\}$ into global (instance-level) variables $\z^{\textsc{g}}$ and local (datapoint-level) variables $\z^{\textsc{l}}$. 
The dimensionality of global variables is typically constant, whereas local variables $\z^\textsc{l} = \{\z^\textsc{l}_1, \ldots, \z^\textsc{l}_N\}$ have a dimensionality that increases with the data size $N$. For models with this structure, the local variables are typically conditionally independent $z^\textsc{L}_n \bot z^\textsc{L}_{-n} \mid x, z^\textsc{g}$, which means that we can parameterize the variational distributions as
\begin{align*}
    \tilde{\lambda}^\textsc{g}
    &= 
    \lambda^\textsc{g} 
    + 
    \sum_{n=1}^N 
    T^\textsc{g}_\f(\x_n, z^\textsc{l}_{n}),
    &
    \tilde{\lambda}^\textsc{l}_n
    &=
    \lambda^\textsc{l}_n + T^\textsc{l}_\f(\x_n, \z^\textsc{g}).
\end{align*}
The advantage of this parameterization is it allows us to train approximate Gibbs updates for global variables in a manner that scales dynamically with the size of the dataset, and appropriately adjusts the posterior variance according to the amount of available data. See Appendix~\ref{appendix:architecture} for detailed explanations of how we use this parameterization to design the neural proposals in our experiments.
\begin{figure*}[t!]
  \centering
  \begin{subfigure}[t]{0.5\textwidth}
  \includegraphics[width=85mm]{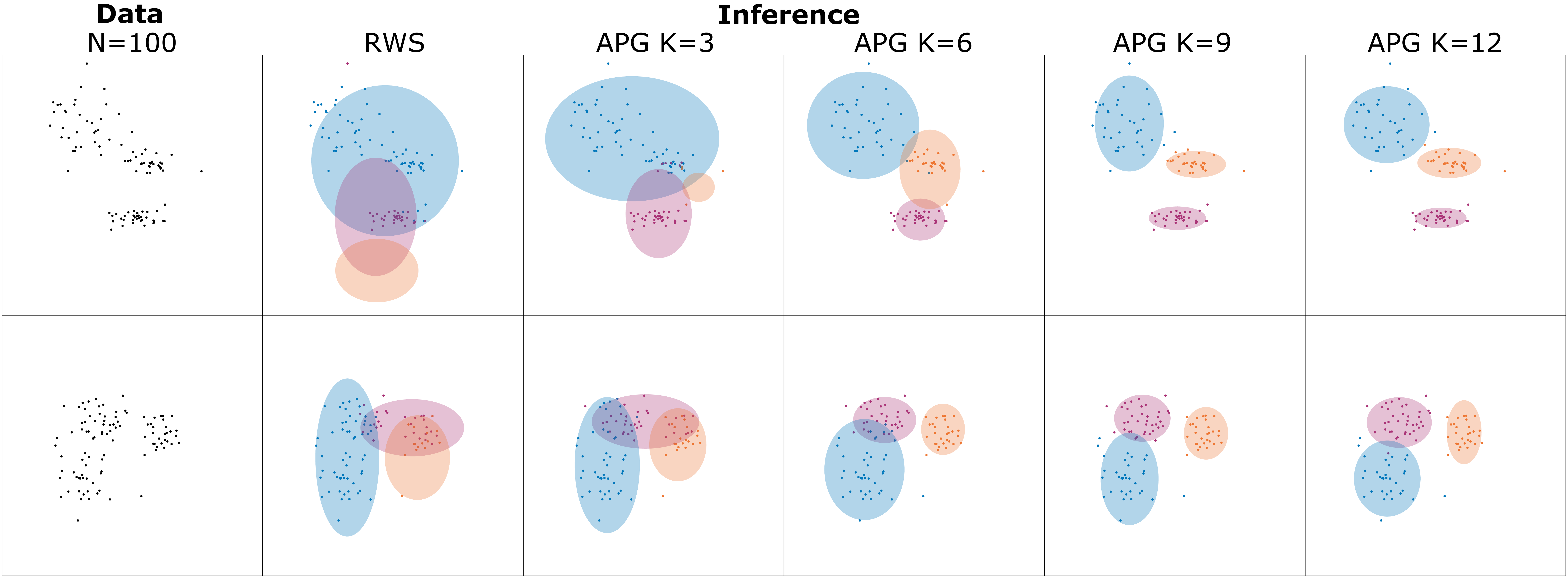}
  \caption{GMM}
  \label{fig:samples-gmm}
  \end{subfigure}%
  \begin{subfigure}[t]{0.5\textwidth}
  \includegraphics[width=85mm]{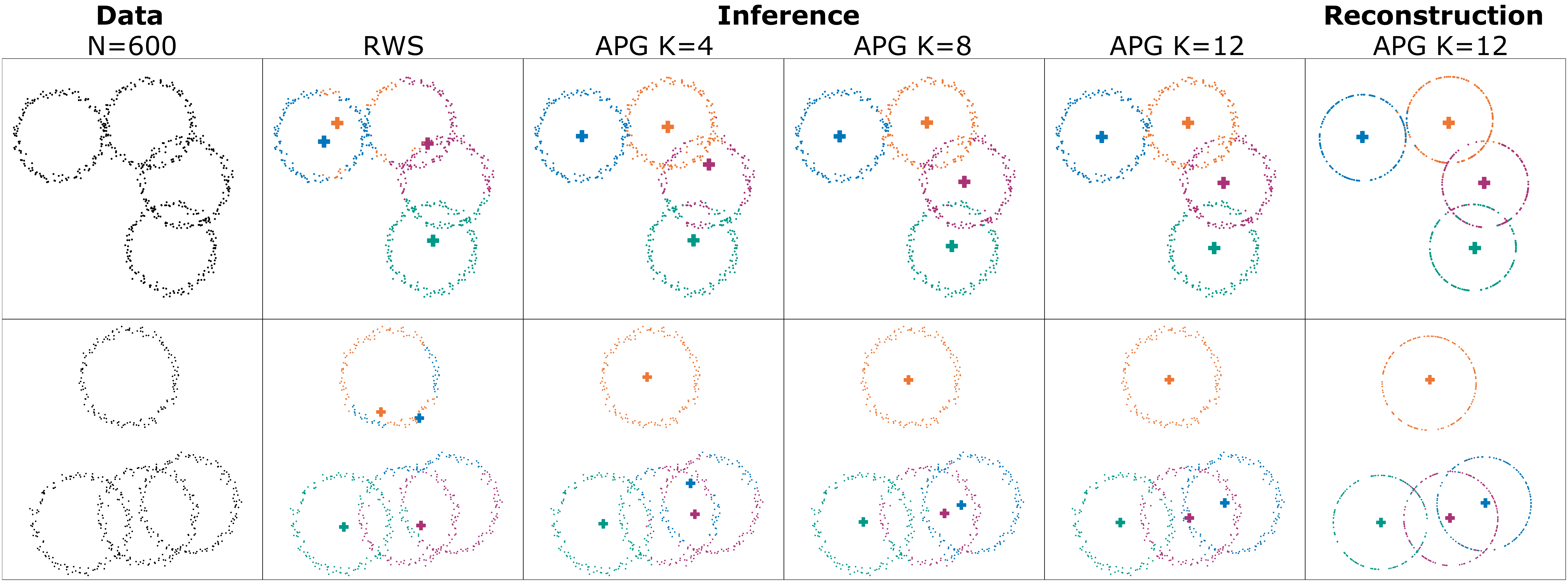}
  \caption{DMM}
  \label{fig:samples-dgmm}
  \end{subfigure}
  \caption{Visualization of a single sample for (\textbf{a}) GMM instances with $N=100$ data points and (\textbf{b}) DMM instances with $N=600$ data points. The left columns show the data instance, the second column the inference result from RWS, and the subsequent columns the inference result after $K$ APG sweeps. In addition the rightmost column in (\textbf{b}) shows the reconstruction from the generative model.
  Unlike RWS, APG successfully learns the model parameters and an inference model and its performance improves with increasing $K$.}
  \label{fig:samples-mixture}
\end{figure*}

\section{Related Work}

Our work fits into a line of recent methods for deep generative modeling that seek to improve inference quality, either by introducing auxiliary variables~\cite{maaloe2016auxiliary, ranganath2016hierarchical}, or by performing iterative updates~\cite{marino2018iterative}. 
Work by \citet{hoffman2017learning} applies Hamiltonian Monte Carlo to samples that are generated from the encoder, which serves to improve the gradient estimate w.r.t.~$\theta$ (Equation~\ref{eq:grad-theta}), while learning the inference network using a standard reparameterized ELBO objective. 
\citet{li2017approximate} also use MCMC to improve the quality of samples from an encoder, but additionally use these samples to train the encoder by minimizing the inclusive KL divergence relative to the filtering distribution of the Markov chain. As in our work, the filtering distribution after multiple MCMC steps is intractable. \citet{li2017approximate} therefore use an adversarial objective to minimize the inclusive KL. 
Neither of these methods consider block decompositions of latent variables, nor do they learn kernels. 


\citet{salimans2015markov} derive a stochastic lower bound for variational inference which uses an importance weight defined in terms of forward and reverse kernels in MCMC steps, similar to Equation~\ref{eq:incremental-weight-forward-reverse}.
\citet{caterini2018hamiltonian} extend this work by optimally selecting reverse kernels (rather than learning them) using inhomogeneous Hamiltonian dynamics.
\citet{huang2018improving} learn a sequence of transition kernels that performs annealing from the initial encoder to the posterior.
\citet{ruiz2019contrastive} define a contrastive divergence which can be tractably optimized w.r.t~variational parameters $\phi$.
Since all of these methods minimize an exclusive KL, rather than an inclusive KL, the gradient estimates rely on reparameterization, which makes them inapplicable to models with discrete variables. 
Moreover, these methods perform a joint update on all variables, while we consider the task of learning conditional proposals.


\citet{wang2018meta} develop a meta-learning approach to learn Gibbs block conditionals. This work assumes a setup in which it is possible to sample $x, z$ from the true generative model $p(x,z)$. This circumvents the need for estimators in Equation~\ref{eq:grad-self-normalized} and~\ref{eq:grad-theta}, which are necessary when we wish to learn the generative model. Similar to our work, this approach minimizes the inclusive KL, but uses the learned conditionals to define an (approximate) MCMC sampler, rather than using them as proposals in a SMC sampler. This work also has a different focus from ours, in that it primarily seeks to learn block conditionals that have the potential to generalize to previously unseen graphical models.

\section{Experiments}
\label{sec:experiments}

We evaluate APG samplers on three different tasks. 
We begin by analysing APG samplers on a Gaussian mixture model (GMM) as an exemplar of a model in the conjugate-exponential family. 
This experiment allows us to analytically compute the conditional posterior to verify if the learned proposals indeed converge to the Gibbs kernels.
Here APG samplers outperform other iterative inference methods even when using a smaller computational budget. 
Next we consider a deep generative mixture model (DMM) that incorporates a neural likelihood and use APG samplers to jointly train the generative and inference model. We demonstrate the scalability of APG samplers by performing accurate inference at test-time on instances with 600 points. 
In our third experiment, we consider an unsupervised tracking model for multiple bouncing MNIST digits. We extend the task proposed by \citet{srivastava2015unsupervised} to up to five digits, and learn both a deep generative model for videos and an inference model that performs tracking. At test time we show that APG easily scales beyond previously reported results for a specialized recurrent architecture \cite{kosiorek2018sequential}. The results for each of these tasks constitute significant advances relative to the state of the art. 
APG samplers are not only able to scale to models with higher complexity, but also provides a general framework for performing inference in models with global and local variables, which can be adapted to a variety of model classes with comparative ease.

\subsection{Baselines}
\label{sec:baselines}

We compare our APG sampler with three baseline methods. The first is RWS. The second is a bootstrapped population Gibbs (BPG) sampler, which uses exactly the same sampling scheme as in Algorithm~\ref{alg:amortized-gibbs}, but proposes from the prior rather than from  approximate Gibbs kernels. This serves  to evaluate whether learning proposals improves the quality of inference results. The third is a method that augments RWS with Hamiltonian Monte Carlo (HMC) updates, which is analogous to the method proposed by \citet{hoffman2017learning}. This serves to compare Gibbs updates to those that can be obtained with state-of-the-art MCMC methods.


In the HMC-RWS baseline, we employ the standard RWS gradient estimators from Equations~\ref{eq:grad-phi-rws} and~\ref{eq:grad-theta}, but update samples from the standard encoder using HMC. In the bouncing MNIST model, all variables are continuous and are updated jointly. In the GMM, we use HMC to sample from the marginal for continuous variables (by marginalizing over cluster assignments via direct summation), and then sample discrete variables from enumerated Gibbs conditionals to evaluate the full log joint $\log p_\theta(x,z)$. For the DMM we perform the same procedure, but use HMC to sample from the conditional distribution given discrete variables, because direct summation is not possible in this model. 


Unless otherwise stated, we compare methods under an equivalent computational budget. We compare an APG sampler that performs $K$ sweeps with $L$ particles to RWS with $K \cdot L$ particles. When comparing to HMC-RWS methods, we perform $K$ updates for $L$ particles with a number of leapfrog steps $\text{LF in \{1,5,10\}}$, which corresponds to a 1x, 5x, and 10x computational budget as measured in terms of the number of evaluations of the log joint. We resort to report the log joint $\log p_\q(\x, z)$ as an evaluation metric, because the marginal $q_\f(z^k | x)$ is intractable. A summary of these results can be found in Table~\ref{table:log-joint-all}.

\subsection{Gaussian Mixture Model}
\label{sec:gmm}
To evaluate whether APG samplers can learn the exact Gibbs updates in conditionally conjugate models, we consider a 2D Gaussian mixture model 
\begin{align*}
    \mu_m, \tau_m &\sim \text{Normal-Gamma}(\mu_0, \nu_0, \alpha_0, \beta_0)
    , \\
    c_n &\sim \mathrm{Cat}(\pi), \,
    x_n | c_n \,=\, m \sim \text{Normal}(\mu_m, 1 / \tau_m)
    \\
    m &= 1,2..,M, n = 1,2,..,N
    .
\end{align*}
Each cluster has a vector-valued mean $\mu_m$ and a diagonal precision $\tau_m$, for which we define an elementwise NormalGamma prior with $\mu_0=0, \nu_0=0.1, \alpha=2, \beta=2$.

In this model, the global variables $z^\textsc{g} = \{\mu_{1:M}, \tau_{1:M}\}$ are means and precisions of $M$ mixture components; The local variables $z^\textsc{l} = \{c_{1:N}\}$ are cluster assignments of $N$ points. Conditioned on cluster assignments, the Gaussian likelihood is conjugate to a normal-gamma prior with pointwise sufficient statistics $T(x_n, c_n)$
\begin{align*}
    \Big\{\mathrm{I}[c_n \!=\! m], 
        ~\mathrm{I}[c_n \!=\! m] \, x_n, 
        ~\mathrm{I}[c_n \!=\! m] \, x_n^2 
        ~\Big\vert~m \!=\! 1,2,\dots,M 
    \Big\}
\end{align*}
where the identity function $\mathrm{I}[c_n \!=\! m]$ evaluates to 1 if the equality holds, and 0 otherwise.

We employ a block update strategy that iterates between blocks $\{\mu_{1:M}, \tau_{1:M}\}, \{c_{1:N}\}$ by learning neural proposals $q_\f(\mu, \tau \mid x, c)$ and $q_\f(c \mid x, \mu, \tau)$.
We use pointwise neural sufficient statistics modeled based on the ones in the analytic form.
We train our model on 20,000 GMM instances, each containing $M = 3$ clusters and $N = 60$ points; We train with $K=5$ sweeps, $L=10$ particles, $20$ instances per batch, learning rate $2.5\times10^{-4}$, and $2\times10^5$ gradient steps.
\begin{figure}[t!]
\centering
\includegraphics[width=\columnwidth]{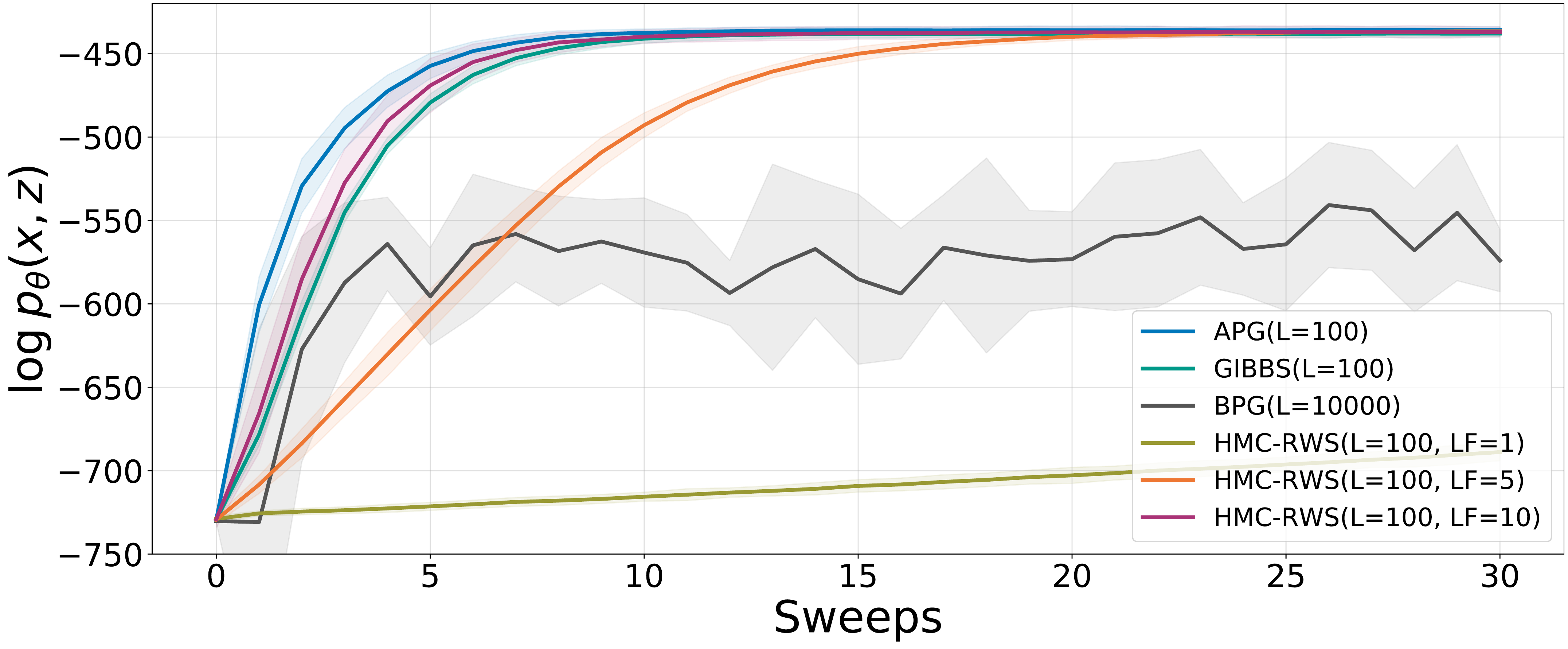}
  \caption{$\log p_\theta(x, z)$ as a function of number of sweeps for a GMM test instance. We perform each method with $L$ particles. For HMC-RWS, we perform multiple leapfrog steps \text{LF} in $\{1,5,10\}$. To achieve results comparable to APG and Gibbs samplers, HMC-RWS needs 10 times the computational budget, while BPG underperforms even with 100 times the computational budget.}
  \label{fig:convergence-gmm}
  \vspace{-0.75em}
\end{figure}

Figure~\ref{fig:samples-gmm} shows single samples in 2 test instances, each containing $N=100$ points. Even when using a parameterization that employs neural sufficient statistics, the RWS encoder fails to propose reasonable clusters, whereas the APG sampler converges within 12 sweeps in GMMs with more variables than training instances.

We compare the APG samplers with true Gibbs samplers and the baselines in section~\ref{sec:baselines}. Figure~\ref{fig:convergence-gmm} shows that the APG sampler converges to the $p_\q(x, z)$ achieved the true Gibbs sampler; it outperforms HMC-RWS even when we run it with 10 times the computation budget; and learned proposals substantially improve on inference since BPG does not converge even with 100x computational budget. In addition, we verify the approximate Gibbs kernel converges to the true Gibbs kernel (see Appendix~\ref{appendix-kl-gmm-training}).

\begin{figure}[!t]
\centering
\includegraphics[width=\columnwidth]{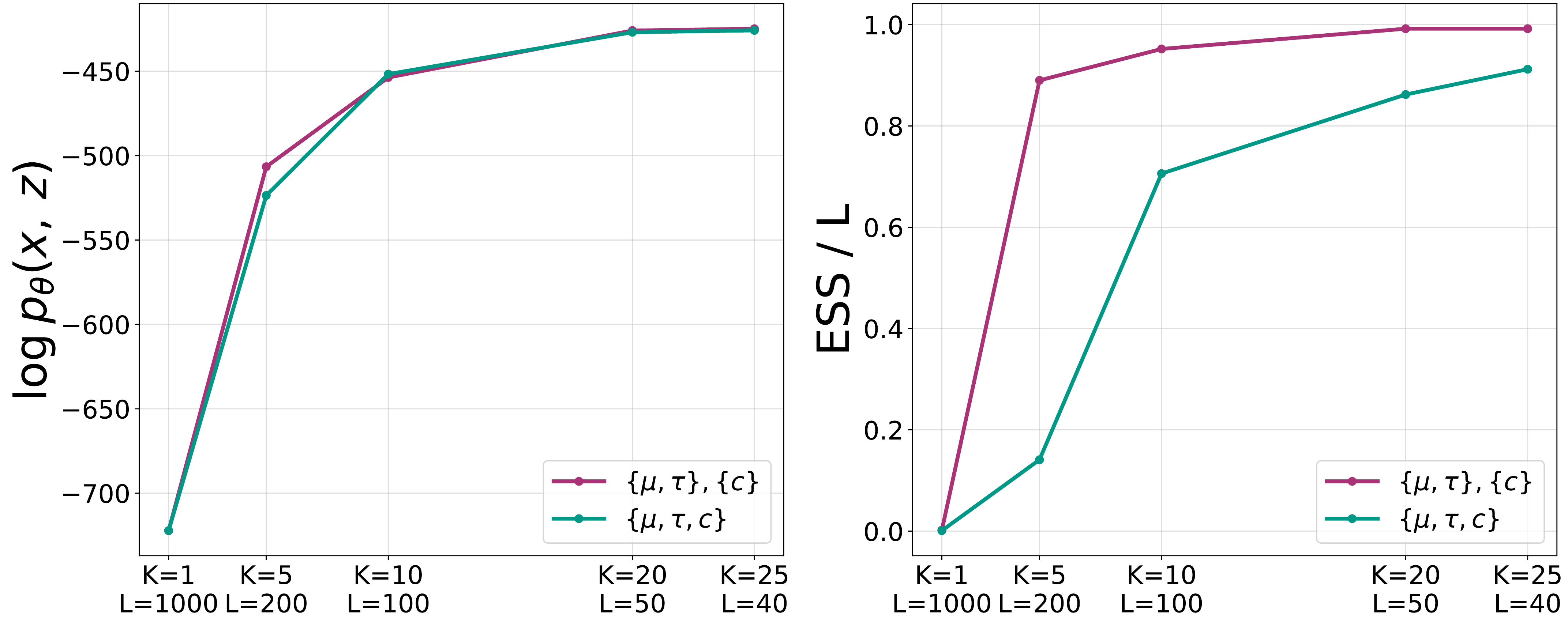}
  \caption{$\log p(x, z)$ (left) and ESS (right) as a function of number of sweeps for APG samplers with different block update strategies and same overall sampling budget $K \cdot L = 1000$. A higher ESS is achieved by performing more sweeps with fewer samples.}
  \label{fig:fixed-budget-gmm}
  \vspace{-0.75em}
\end{figure}

\begin{table*}[t!]
    \centering
    \vspace{-1em}
    \caption{$\log p_\q (x, z)$ averaged over test corpus for RWS, BPG, HMC-RWS, and Gibbs Sampling for different numbers of sweeps $K$.
    We run HMC-RWS with 1, 5, and 10 leapfrog integration steps, which corresponds to 1, 5, and 10 times the computational budget used by APG with the same number of sweeps.
    The GMM test corpus contains $20,000$ instances, each of which has $N=100$ data points. The DMM test corpus contains $20,000$ instances with $N=600$ data points per instance. Bouncing MNIST's test corpus contains $10,000$ instances each containing $T=100$ time steps and $D=5$ digits.
    In all tasks APG with $K=20$ performs the best.}
    \vspace{0.5em}
    \addtolength{\tabcolsep}{-2pt} 
    \begin{tabularx}{\textwidth}{cccccccccc}
    \toprule
     & RWS & BPG & HMC-RWS & HMC-RWS & HMC-RWS & APG & APG & APG & GIBBS \\
     & & K=20 & K=20, LF=1 & K=20, LF=5 & K=20, LF=10 & K=5 & K=10 & K=20 & K=20 \\
    \midrule
    GMM ($\times10^2$)& $-7.651$ & $-6.688$ & $-6.594$ & $-5.420$ & $-4.673$ & $-4.816$ & $-4.556$ & $\textbf{-4.469}$ & $-4.573$ \\
    DMM ($\times10^3$) & $-3.172$ & $-4.204$ & $-2.185$ & $-2.184$ & $-2.184$ & $-2.050$ & $-2.040$ & $\textbf{-2.035}$ & -- \\
    BC-MNIST ($\times10^5$) & $-1.673$ & $-1.884$ & $-1.538$ & $-1.405$ & $-1.303$ & $-0.706$ & $-0.652$ & $\textbf{-0.620}$ & --\\
    \bottomrule
    \vspace{-1.25em}
\end{tabularx}
\label{table:log-joint-all}
\end{table*}

\begin{figure*}[!t]
  \centering
  \includegraphics[width=1.0\textwidth]{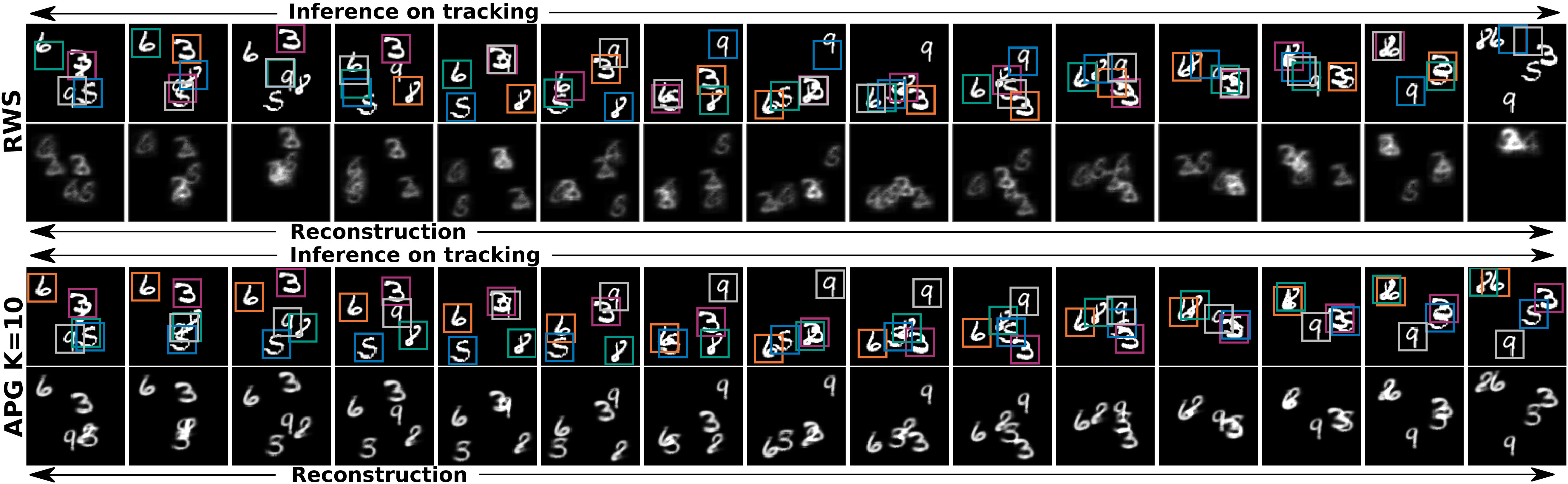}
  \caption{A single sample from variational distribution in one video instance with $T=100$ time steps and $D=5$ digits. The visualization is truncated to the first 15 time steps due to limited space (see Appendix~\ref{appendix:full-recons} for full time series and Appendix~\ref{appendix:bmnist-comparison-rws} for more examples). The sample is initialized from RWS (top) and updated by $K=10$ APG sweeps (bottom). The APG sampler improves the inference results in tracking and learns better MNIST digit embeddings that improve reconstruction.}
  \label{mnist-qualitative}
\end{figure*}

Moreover, we compare joint block updates $\{\mu, \tau, c\}$ with decomposed blocks $\{\mu, \tau\}, \{c\}$, for varying number of sweeps (see Figure~\ref{fig:fixed-budget-gmm}). 
We additionally assess sample quality using the effective sample size (ESS)
\begin{align*}
    \frac{\text{ESS}}{L} 
    = 
    \frac{(\sum_{l=1}^L w^{k,l})^2}
         {L \sum_{l=1}^L (w^{k,l})^2}
    .
\end{align*}
We can see that it is more effective to perform more sweeps $K$ with a smaller number of particles $L$ for a fixed computation budget. We also see that the decomposed block updates result in higher sample quality.


\subsection{Deep Generative Mixture Model}

In the deep generative mixture model (DMM) $p_\q(\x, \z)$ our data is a corpus of 2D ring-shaped mixtures. The generative model takes the form
\begin{align*}
    \mu_m &\sim \text{Normal}(\mu, 
    \sigma_0^2 \, I), 
     \\
    c_n &\sim \mathrm{Cat}(\pi), h_n \sim \text{Beta}(\alpha, \beta),  \\
    x_n \,|\, c_n=m &\sim \text{Normal}(g_\q(h_n) + \mu_m, \sigma^2_{\epsilon} I)
    \\
    m &= 1, 2, ..., M, \:, \: n = 1, 2, ..., N
    .
\end{align*}
The global variables $z^\textsc{g} = \{\mu_{1:M}\}$ are the centers of $M$ components; The local variables $z^\textsc{l} = \{c_{1:N}, h_{1:N}\}$ are cluster assignments and 1D embeddings of $N$ points.  Given an assignment $c_n$ and an embedding $h_n$ that goes into a MLP decoder $g_\q$, we sample a data point $x_n$ with Gaussian noise. We choose $\mu=0, \sigma_0=10, \alpha=1, \beta=1, \sigma_\epsilon=0.1$. 

We employ a block update strategy that iterates between blocks $\{\mu_{1:M}\}$, $\{c_{1:N}, h_{1:N}\}$ by learning neural proposals $q_\f(\mu \mid x, c, h)$ and $q_\f(c, h \mid x, \mu)$.
We train our model on 20,000 instances, each containing $M = 4$ clusters and $N=200$ points; We train our model with $K=8$ sweeps, $L=10$ particles, 20 instances per batch, learning rate $10^{-4}$, and $3\times10^5$ gradient steps.

Once again, we compare the APG sampler with the encoders using RWS. Figure~\ref{fig:samples-dgmm} shows visualization of single samples in 2 test instances, each containing $N=600$ points. In both mixture model experiments, we can tell that the APG samplers scale to a much larger number of variables, whereas a standard encoder trained using RWS fails to produce reasonable proposals.

\subsection{Time Series Model -- Bouncing MNIST}

In the bouncing MNIST model, our data is a corpus of video frames that contain multiple moving MNIST digits. To demonstrate the scalability of APG samplers, we train both a deep generative model and an inference model with $600,000$ video instances, each of which contains $10$ time steps and $3$ digits. At test time we evaluate the model on instances that contain up to $100$ time steps and $5$ digits. 

Consider a sequence of video frames $x_{1:T}$, which contains $T$ time steps and $D$ digits. We assume that the $k$th digit in the $t$th frame $x_t$ can be represented by some digit feature $z^{\mathrm{what}}_{d}$ and a time-dependent position variable $z^{\mathrm{where}}_{d, t}$. The deep generative model takes the form
\begin{align*}
    z^{\mathrm{what}}_{d} 
    \,&\sim \,
    \text{Normal}(0, \, I)
    ,
    \\
    z^{\mathrm{where}}_{d, 1} \,&\sim\, \text{Normal}(0,\, I)
    ,
    \,
    z^{\mathrm{where}}_{d, t} 
    \sim 
    \text{Normal}(z^{\mathrm{where}}_{d, t - 1}, \, \sigma^2_0 I),
    \\
    x_t 
    &\sim
    \mathrm{Bernoulli}
    \Big(
        \sigma
        \Big(
            \sum_{d} \mathrm{ST}
            \big(
                g_\q (\z_{d}^{\mathrm{what}})
                , 
                \:
                z^{\mathrm{where}}_{d, t}
            \big)
        \Big)
    \Big)
    \\
    d &= 1, 2, ..., D, t = 1, 2,.., T
    .
\end{align*}
In this model, the global variables $z^\textsc{g} = \{\z_{1:D}^{\mathrm{what}}\}$ are latent codes of the $D$ digits, as those digits don't change across the frames; The local variables $z^\textsc{l} = \{\z_{1:D, 1:T}^{\mathrm{where}}\}$ are time-dependent positions. The hyper-parameter is $\sigma_0=0.1$. The dimensionalities are $z^{\mathrm{what}}_{d} \in \mathbb{R}^{10}$, $z^{\mathrm{where}}_{d, t} \in \mathbb{R}^2$.

\begin{table}[!t]
    \centering
    \caption{Bouncing MNIST performance. Mean-squared error between test instance and its reconstruction across 5000 video instances with different number of time steps $T$ and digits $D$. We use a fixed sampling budget of $K \cdot L=1000$. APG samplers achieve substantial improvements relative to RWS and HMC-RWS.}
    \begin{tabularx}{\columnwidth}{ccccc}
    \toprule
        & RWS & HMC-RWS & APG & APG \\
        &  &  K=20, LF=10 & K=10 & K=20 \\
    \midrule
    D=5, T=20 & 265.1 & 261.2 & 117.7 & \textbf{103.8}\\
    \hspace{0.5em}D=5, T=100 & 272.1 & 267.9  & 115.4 & \textbf{102.3}\\
    D=3, T=20 & 134.0 & 124.1 & 60.4 & \textbf{58.9} \\   
    \hspace{0.5em}D=3, T=100 & 137.3 & 127.4 & 56.0 & \textbf{55.7} \\
    \bottomrule
    \end{tabularx}
    \label{table:mse-bmnist}
\end{table}

Given each digit feature $\z_{d}^{\mathrm{what}}$, a MLP decoder $g_\q$ generates a MNIST digit image $g_\q (\z_{d}^{\mathrm{what}})$ of the size 28$\times$28. Then a spatial transformer \cite{jaderberg2015spatial} ST will map each digit image onto a 96$\times$96 canvas using the corresponding position $z^{\mathrm{where}}_{d, t}$. Finally we sample the frame $x_t$ from a Bernoulli distribution, the parameter of which is the composition of that canvas with some activation function $\sigma$.

Compared with joint prediction $\{z_d^{\mathrm{what}}, z_{d, 1:T}^{\mathrm{where}}\}$, it is easier to guess the features $\{z_d^{\mathrm{what}}\}$ of digit $d$ given its positions $\{z_{d, 1:T}^{\mathrm{where}}\}$ and vice versa. Moreover, we can predict the positions in a step-by-step manner and resample at each sub-step of $\{z_{d, t}^{\mathrm{where}}\}$, which will reduce the variance of the importance weights. Thus we employ a block update strategy that iterates over $T + 1$ blocks
\begin{align*}
    \{z_{1:D}^{\mathrm{what}}\}, \{z_{1:D, 1}^{\mathrm{where}}\}, \{z_{1:D, 2}^{\mathrm{where}}\}, \dots, \{z_{1:D, T}^{\mathrm{where}}\}.
\end{align*}
In each block we sequentially predict variables for each digit $d=1, 2, \dots, D$ (see Appendix~\ref{appendix:architecture} for detail). 

We train our model on 60000 bouncing MNIST instances, each containing $T=10$ time steps and $D=3$ digits, with $K=5$ sweeps, $L=10$ particles, $5$ instances per batch, learning rate $10^{-4}$, and $1.2\times10^7$ gradient steps.

Figure~\ref{mnist-qualitative} shows that the APG sampler substantially improves inference and reconstruction results over RWS encoders, and scales to test instances with hundreds of latent variables. Moreover we compute the mean squared error between the video instances and their reconstructions. Table~\ref{table:mse-bmnist} shows that APG sampler can substantially improve the sample from RWS encoder, and scales up to 100 time steps and 5 digits.


\section{Conclusion}
We introduce APG samplers, a framework for amortized variational inference based on adaptive importance samplers that iterates between updates to blocks of variables. To appropriately account for the size of the input data we developed a novel parameterization in terms of neural sufficient statistics inspired by sufficient statistics in conjugate exponential families.
We show that APG samplers can train structured deep generative models with hundreds of instance-level variables in an unsupervised manner, and compare favorably to existing amortized inference methods in terms of computational efficiency.

APG samplers offer a path towards the development of deep generative models that incorporate structured priors to provide meaningful inductive biases in settings where we have little or no supervision. These methods have particular strengths in problems with global variables, but more generally make it possible to design amortized approaches that exploit conditional independence. Our parameterization in terms of neural sufficient statistics makes it comparatively easy to design models that scale to much larger number of variables and thus generalize to datasets that vary in size. 


\section*{Acknowledgements}
We would like to thank our reviewers and the area chair
for their thoughtful comments. This work was supported by the Intel Corporation, NSF award 1835309, startup funds from Northeastern University, the Air Force Research Laboratory (AFRL), and DARPA. Tuan Anh Le was supported by AFOSR award FA9550-18-S-0003.

\bibliography{icml-2020-references}
\bibliographystyle{icml2020}

\newpage
\appendix
\onecolumn
\icmltitle{Supplementary Material: Amortized Population Gibbs Samplers\\ with Neural Sufficient Statistics}
\icmlkeywords{Machine Learning, ICML}
\section{Gradient of the generative model}
\label{appendix:grad-theta}
We show that the gradient of the marginal $\nabla_\q \, \log \, p_\q(x)$ can be estimated using self-normalized importance sampling. First of all, we express the expected gradient of the log joint as
\begin{align*}
    \mathbb{E}_{p_\q(\z | \x)} 
    \left[
    \nabla_\q \log p_\q(\x, \z)
    \right]
    =
    \mathbb{E}_{p_\q(\z | \x)} 
    \left[
    \nabla_\q \log p_\q(\x) + \nabla_\q \log p_\q(\z | \x)
    \right]
    =
    \mathbb{E}_{p_\q(\z | \x)} 
    \left[
    \nabla_\q \log p_\q(\x) 
    \right]
    =
    \nabla_\q \log p_\q(\x)
\end{align*}
Here we make use of a standard identity that is also used in likelihood-ratio estimators
\begin{align*}
\mathbb{E}_{p_\q(\z | \x)}
\left[
    \nabla_\q \log p_\q(\z | \x)
\right] 
=
\int p_\q(\z | \x) \nabla_\q \log p_\q(\z | \x) \: dz
=
\int \nabla_\q p_\q(\z | \x) \: dz
=
\nabla_\q \int p_\q(\z | \x) \: dz
=
\nabla_\q 1
= 
0
\end{align*}
Therefore, we have the the following equality
\begin{align*}
\nabla_\q \log p_\q(\x) 
= 
\mathbb{E}_{p_\q(\z | \x)} 
\left[
\nabla_\q \log p_\q(\x, \z)
\right]
\simeq
\sum_{l=1}^L
\frac{w^l}{\sum_{l'} w^{l'}}
\nabla_\q
\log p_\q(x, z^l).
\end{align*}
which is the self-normalized gradient estimator in Equation ~\ref{eq:grad-theta}.

\section{Importance weights in sequential importance sampling}
\label{appendix:sis-weight}
We will prove the form of importance weight $w^k$ in sequential importance sampling. At step $k=1$, we use exactly the standard importance sampler, thus it is obvious that the following is a valid importance weight
\begin{align*}
    w^1 = \frac{\gamma^1(z^1)}{q^1(z^1)}.
\end{align*}
When step $k>2$, we are going to prove that the importance weight relative to the intermediate densities has the form
\begin{align}
    \label{appendix:eq:sis-weight}
    w^k
    = 
    \frac{\gamma^k(z^{1:k})}
         {q^1(z^1) \prod_{k'=2}^k q^{k'}(z^{k'} \mid z^{1:k'-1})}.
\end{align}

At step $k=2$, the importance weight is defined as 
\begin{align*}
    w^k 
    &= 
    v^{2} \: w^1
    \: =
    \frac{\gamma^2(z^{1:2})}{\gamma^{1}(z^{1})\:q^2(z^2 \mid z^{1})} \frac{\gamma^1(z^1)}{q^1(z^1)}
    \: = \frac{\gamma^2(z^{1:2})}{q^1(z^1) \: q^2(z^2 \mid z^{1})}.
\end{align*}
which is exactly that form. Now we prove weights in future steps by induction. At step $k\geq 2$, we assume that the weight has the form in Equation~\ref{appendix:eq:sis-weight}, i.e.
\begin{align*}
    w^k
    = 
    \frac{\gamma^k(z^{1:k})}
         {q^1(z^1) \prod_{k'=2}^k q^{k'}(z^{k'} \mid z^{1:k'-1})}.    
\end{align*}
then at step $k+1$, the importance weight is the product of incremental weight and incoming weight 
\begin{align*}
    w^{k+1}
    =
    v^{k+1} \: w^k
    =
    \frac{\gamma^{k+1}(z^{1:k+1})}{\gamma^{k}(z^{1:k})\:q^{k+1}(z^{k+1} \mid z^{1:k})}
    \frac{\gamma^k(z^{1:k})}
         {q^1(z^1) \prod_{k'=2}^k q^{k'}(z^{k'} \mid z^{1:k'-1})}
    =
    \frac{\gamma^{k+1}(z^{1:k+1})}{q^1(z^1) \prod_{k'=2}^{k+1} q^{k'}(z^{k'} \mid z^{1:k'-1})}
    .    
\end{align*}
Thus the importance weight $w^k$ has the form of Equation~\ref{appendix:eq:sis-weight} at each step $k>2$ in sequential importance sampling.
\section{ Derivation of Posterior Invariance}
\label{appendix:posterior-invariance}
We consider a \emph{sweep} of conditional proposals at step $K$ as
\begin{align}
    \label{eq:approx-gibbs-kernel}
    p_\q \big( \z^k \mid \x, \z^{k-1} \big)
    &=
    \prod_{b=1}^B
    p_\q (\z^k_b \mid \x, \z^{k}_{\prec b}, \z^{k-1}_{\succ b})
    ,
\end{align}
where $\z_{\prec b} = \{z_i \mid i < b\}$ and $\z_{\succ b} = \{z_i \mid i > b\}$. Additionally we define $\z_{\preceq b} = \{z_i \mid i \leq b\}$. 

We will show that any partial update within a sweep, i.e.
\begin{align}
    p_\q \big( \z^k_{\preceq b} \mid \x, \z^{k-1} \big)
    &=
    \prod_{v=1}^b
    p_\q (\z^k_v \mid \x, \z^{k}_{\prec v}, \z^{k-1}_{\succ v})
    ,
    \quad
    \forall b\in\{1, 2,..., B\}
\end{align}
will leave the posterior invariant. In fact, for any choice of $b$ we have
\begin{align*}
    \int 
    \:
    p_\q(\z^{k-1} \,\mid\, \x) 
    \: 
    p_\q(\z^k_{\preceq b} \,\mid\, x,\, \z^{k-1}) 
    \:
    dz^{k-1}_{\preceq b} 
    &= 
    \int 
    \:
    p_\q(\z^{k-1}_{\preceq b}\,, \z^{k-1}_{\succ b} \,\mid\, \x) 
    \: 
    dz^{k-1}_{\preceq b} 
    \prod_{v=1}^b p_\q(\z^k_v \mid \x, \z^{k}_{\prec v}, \z^{k-1}_{\succ v})
    \\
    &=
    p_\q(\z^{k-1}_{\succ b} \,\mid\, \x)
    \:
    p_\q(\z^k_{\preceq b} \mid \x, \z^{k-1}_{\succ b})
    \\
    &=
    p_\q(\z^k_{\preceq b}\,, \, \z^{k-1}_{\succ b} \mid \x)
    .
\end{align*}
When we require the APG proposal $q_\f(z'_b \,|\, x, z_{-b})$ leaves the posterior invariant (by minimizing the inclusive KL divergence relative to the conditional posterior $p_\q(z_b \,|\, x, \z_{-b})$), then any sweep or part of one sweep will also leave the posterior invariant, as what we prove above. This means that at test time we can apply arbitrary number of APG sweeps, each of which will results in samples that approximate the posterior $p_\q(z \,|\, x)$.

\section{Resampling Algorithm}

\label{appendix:resample-algo}
\begin{algorithm}[!h]
    \setstretch{1.2}
  \caption{Multinomial Resampler}
  \label{alg:resample}
\begin{algorithmic}[1]
  \State \textbf{Input} Weighted samples $\{z^l, w^l \}_{l=1}^L$
  \For {$i = 1$ \textbf{to} $L$}
    \State $a^i \sim \mathrm{Discrete}(\{w^l / \sum_{l' = 1}^L w^{l'}\}_{l=1}^L)$\Comment{Index Selection} 
    \State Set Set $ \tilde{z}^{\:i} = z^{a^i}$
    \State Set $ \tilde{w}^{\:i} = \frac{1}{L} \sum_{l = 1}^L w^l$ \Comment{Re-weigh}
    \EndFor
  \State \textbf{Output} Equally weighted samples $\{\tilde{z}^l, \tilde{w}^l \}_{l=1}^L$
\end{algorithmic}
\end{algorithm}

\section{Proof of the amortized population Gibbs samplers algorithm}
\label{appendix:proof-algo}

Here, we provide an alternative proof of correctness of the APG algorithm given in Algorithm~\ref{alg:amortized-gibbs}, based on the construction of proper weights~\cite{naesseth2015nested} which was introduced after SMC samplers~\cite{delmoral2006sequential}.
In section~\ref{appendix-proof-pw-begin-section}, we will introduce proper weights; In section~\ref{appendix-proof-pw-end-section}, we then present several operations that preserve the proper weighting property; In section~\ref{appendix-proof-apg-section}, we will take use of these properties to prove the correctness of APG samplers algorithm (Algorithm~\ref{alg:amortized-gibbs}).

\subsection{Proper weights}
\label{appendix-proof-pw-begin-section}
\begin{definition}[Proper weights]
    Given an unnormalized density $\tilde p(z)$, with corresponding normalizing constant $Z_p := \int \tilde p(z) \,\mathrm dz$ and normalized density $p \equiv \tilde p / Z_p$, the random variables $z, w \sim P(z, w)$ are properly weighted with respect to $\tilde p(z)$ if and only if for any measurable function $f$
    \begin{align}
    \label{eq:pw}
    \E_{P(z, w)}\left[w f(z)\right] = Z_p \E_{p(z)}[f(z)]. 
    \end{align}
    We will also denote this as
    \begin{align*}
        z, w \pw \tilde p.
    \end{align*}
\end{definition}

\paragraph{Using proper weights.}
Given independent samples $z^l, w^l \sim P$, we can estimate $Z_p$ by setting $f \equiv 1$:
\begin{align*}
    Z_p \approx \frac{1}{L} \sum_{l = 1}^L w^l.
\end{align*}
This estimator is unbiased because it is a Monte Carlo estimator of the left hand side of \eqref{eq:pw}.
We can also estimate $\E_{p(z)}[f(z)]$ as
\begin{align*}
    \E_{p(z)}[f(z)] \approx \frac{\frac{1}{L} \sum_{l = 1}^L w^l f(z^l)}{\frac{1}{L} \sum_{l = 1}^L w^l}.
\end{align*}
While the numerator and the denominator are unbiased estimators of $Z_p \E_{p(z)}[f(z)]$ and $Z_p$ respectively, their fraction is biased.
We often write this estimator as
\begin{align}
    \E_{p(z)}[f(z)] \approx \sum_{l = 1}^L \bar w^l f(z^l), \label{eq:pw-estimation}
\end{align}
where $\bar w^l := w^l / \sum_{l' = 1}^L w^{l'}$ is the normalized weight.

\subsection{Operations that preserve proper weights}
\label{appendix-proof-pw-end-section}

\begin{proposition}[Nested importance sampling]
  This is similar to Algorithm 1 in \cite{naesseth2015nested}.
    Given unnormalized densities $\tilde q(z), \tilde p(z)$ with the normalizing constants $Z_q, Z_p$ and normalized densities $q(z), p(z)$, if 
    \begin{align}
        z, w \pw \tilde q, \label{eq:1}
    \end{align}
    then
    \begin{align*}
        z, \frac{w\tilde p(z)}{\tilde q(z)} \pw \tilde p.
    \end{align*}
\end{proposition}
\begin{proof}
    First define the distribution of $z, w$ as $Q$.
    For measurable $f(z)$
    \begin{align*}
        \E_{Q(z, w)}\left[\frac{w\tilde p(z)}{\tilde q(z)} f(z)\right] 
        = Z_q \E_{q(z)}\left[\frac{\tilde p(z) f(z)}{\tilde q(z)}\right]
        = Z_q \int q(z) \frac{\tilde p(z) f(z)}{\tilde q(z)} \,\mathrm dz
        = \int \tilde p\textbf{}(z) f(z) \,\mathrm dz
        = Z_p \E_{p(z)}[f(z)].
    \end{align*}
\end{proof}

\begin{proposition}[Resampling]
\label{proposition:resampling}
  This is similar to Section 3.1 in \cite{naesseth2015nested}.
    Given an unnormalized density $\tilde p(z)$ (normalizing constant $Z_p$, normalized density $p(z)$), if we have a set of properly weighted samples
    \begin{align}
        z^l, w^l \pw \tilde p,  \quad l = 1,\ldots, L \label{eq:bla}
    \end{align}
    then the resampling operation preserves the proper weighting, i.e.
    \begin{align*}
        z'^{\:l}, w'^{\:l} \pw \tilde p, \quad l = 1,\ldots, L
    \end{align*}
    where $z'^{\:l} = z^{a}$ with probability $P(a = i) = w^i / \sum_{l=1}^L w^l$ and $w'^{\:l} := \frac{1}{L} \sum_{l = 1}^L w^l$.
\end{proposition}
\begin{proof}
    Define the distribution of $z^l, w^l$ as $\hat{P}$.
    We show that for any $f$, $\E[f(z^{a}) w'^{\:l}] = Z_p \E_{p(z)}[f(z)]$.
    \begin{align*}
    &
        \E_{\left(\prod_{l=1}^L \hat{P}(z^l, w^l)\right) p(a \mid w^{1:L})} \bigg[f(z^{a}) w'^{l}\bigg] \\
        &= \E_{\prod_{l=1}^L \hat{P}(z^l, w^l)}\left[\sum_{i = 1}^L f(z^i) w' \: P(a = i)\right]  \\
        &= \E_{\prod_{l=1}^L \hat{P}(z^l, w^{l})}\left[\sum_{i = 1}^L f(z^i) w'  \frac{w^i}{ \sum_{l'=1}^L w^{l'}}\right]  \\
        &= \E_{\prod_{l=1}^L \hat{P}(z^l, w^l)}\left[\frac{1}{L}\sum_{i = 1}^L f(z^i) w^i\right] \\
        &= \frac{1}{L}\sum_{i = 1}^L \E_{\hat{P}(z^i, w^i)}\left[f(z^i) w^i\right]
        = \frac{1}{L}\sum_{i = 1}^L Z_p \E_{p(z)}[f(z)]
        = Z_p \: \E_{p(z)}[f(z)]. 
    \end{align*}
\end{proof}
Therefore, the resampling will return a new set of samples that are still properly weighted relative to the target distribution in the APG sampler (Algorithm~\ref{alg:amortized-gibbs}).
\begin{proposition}[Move]
\label{proposition:extendedspace}
    Given an unnormalized density $\tilde p(z)$ (normalizing constant $Z_p$, normalized density $p(z)$) and normalized conditional densities $q(z' \given z)$ and $r(z \given z')$, the proper weighting is preserved if we apply the transition kernel to a properly weighted sample, i.e.~if we have
    \begin{align}
        &
        z^l, w^l \pw \tilde p, \label{eq:pw-of-p}\\[1em]
        &
         z'^{\:l} \sim q(z'^{\:l} \given z^l),  \label{eq:z-prime}\\[1em]
        &
        w'^{\:l} = \frac{\tilde p(z'^{\:l})r(z^l \given z'^{\:l})}{\tilde p(z^l)q(z'^{\:l} \given z^l)} w^l, \qquad l = 1, \ldots, L \label{eq:w-prime}
    \end{align}

then we have
    \begin{align}
        z'^{\:l}, w'^{\:l} \pw \tilde p, \qquad  l = 1, \ldots, L \label{eq:to-prove}
    \end{align}
\end{proposition}
\begin{proof}
    Firstly we simplify the notation by dropping the superscript $l$ without loss of generality. Define the distribution of $z, w$ as $\hat{P}$.
    Then, due to \eqref{eq:pw-of-p}, for any measurable $f(z)$, we have
    \begin{align*}
        \E_P[w f(z)] = Z_p E_{p}[f(z)].
    \end{align*}
    To prove \eqref{eq:to-prove}, we show $\E_{\hat{P}(z, w)q(z' \given z)}[w' f(z')] = Z_p \E_{p(z')}[f(z')]$ for any $f$ as follows:
    \begin{align}
        \E_{\hat{P}(z, w)q(z' \given z)}[w' f(z')]
        &= \E_{\hat{P}(z, w)q(z' \given z)}\left[\frac{\tilde p(z')r(z \given z')}{\tilde p(z)q(z' \given z)} w f(z')\right] 
        \nonumber\\
        &= \int \hat{P}(z, w)q(z' \given z) \frac{\tilde p(z')r(z \given z')}{\tilde p(z)q(z' \given z)} w f(z') \,\mathrm dz \,\mathrm dw \,\mathrm dz' 
        \nonumber\\
        &= \int \hat{P}(z, w) \frac{\tilde p(z')r(z \given z')}{\tilde p(z)} w f(z') \,\mathrm dz \,\mathrm dw \,\mathrm dz' 
        \nonumber\\
        &= \int \tilde p(z') f(z') \left(\int \hat{P}(z, w) w \frac{r(z \given z')}{\tilde p(z)}\,\mathrm dz \,\mathrm dw\right) \,\mathrm dz'
        \nonumber\\
        &= \int \tilde p(z') f(z') Z_p \E_{p(z)}\left[\frac{r(z \given z')}{\tilde p(z)}\right] \,\mathrm dz'. \label{eq:pause}
    \end{align}
    Using the fact that $\E_{p(z)}\left[\frac{r(z \given z')}{\tilde p(z)}\right] = \int p(z) \frac{r(z \given z')}{\tilde p(z)} \,\mathrm dz = \int r(z \given z') \,\mathrm dz / Z_p = 1 / Z_p$.
    Equation~\ref{eq:pause} simplifies to
    \begin{align*}
        \int \tilde p(z') f(z') \,\mathrm dz' = Z_p \E_{p(z')}[f(z')].
    \end{align*}
\end{proof}

\subsection{Correctness of APG Sampler}
\label{appendix-proof-apg-section}
We prove the correctness of the APG sampler (Algorithm~\ref{alg:amortized-gibbs}) by induction. We firstly prove the correctness of the initial proposing step ($k=1$, line~\ref{line:rws-loop} - line~\ref{line:rws-grad-theta}); Then we prove that the algorithm is still correct when we perform one Gibbs sweep step ($k=2$, line~\ref{line:apg-sweep-begin} - line~\ref{line:apg-sweep-end}), given that the previous step is already proved to be correct. By induction we can conclude that its correctness still holds if we perform more Gibbs sweeps.

\textbf{Step $k=1$}. We initialize the proposal of all the blocks $z:=z_{1:B}$ from an initial encoder $z \sim q_\phi(z \given x)$ (line~\ref{line:rws-propose}), which is trained using the wake-$\phi$ phase objective in the standard reweighted wake-sleep\cite{le2019revisiting}, where the objective is 
\begin{align*}
    \E_{\hat{p}(x)}\left[\textsc{kl}\left(p_\theta(z \given x) || q_\phi(z \given x)\right)\right].
\end{align*}
We take gradient w.r.t. variational parameter $\phi$ and compute a self-normalized gradient estimate (line~\ref{line:rws-grad-phi}) as
\begin{align}
    \label{eq:g-rws-phi}
    g_\f :&=
    - \nabla_\phi \: 
    \E_{\hat{p}(x)}
    \left[
    \textsc{kl}\left( p_\theta(z \mid x) \, || \, q_\phi(z \mid  x) \right)
    \right] 
    \\
    &
    =
    \E_{\hat{p}(x)}\left[
    \E_{p_\theta(z \given x)}\left[\nabla_\phi \log q_\phi(z \mid x)\right]
    \right]
    \\
    &
    =
    \sum_{l = 1}^L \frac{w^{l}}{\sum_{l' = 1}^L w^{l'}} \nabla_\f \log q_\phi(z^{l} \mid x)
    ,
    \quad
    z^l \sim q_\f (z \mid x)
    ,
    \quad
    w^{l} = \frac{p_\theta(x, z^{l})}{q_\phi(z^{l} \mid x)}    
    . 
\end{align}
 
Equation~\ref{eq:pw-estimation} will guarantee the validity of this gradient estimate $g_\f$, as long as we show that samples are properly weighted
\begin{align}
    z^l, w^l \pw p_\theta(z, x), \qquad l = 1, \ldots, L.
    \label{eq:invariant}
\end{align}

In fact, $\{(w^l, z^l)\}_{l=1}^L$ are properly weighted because $z^l$ are proposed using importance sampling~\citet{naesseth2015nested}, where $q_\phi(z \given x)$ is the proposal density and $p_\theta(z^l, x)$ is the unnormalized target density. Note that the resampling step (line~\ref{line:resample}) will preserve the proper weighting because of Proposition~\ref{proposition:resampling}.

\textbf{Step $k=2$}. Now we iteratively update each block of the variable $z_b$ for $b = 1, 2, ..., B$, using the corresponding conditional proposal $q_\f (z_b \mid x, z_{-b})$, which is trained by the objective 
\begin{align*}
    \mathcal{K}_b(\f)
    := \,
    \E_{\hat{p}(x)p_\q(z_{-b} | x)}
    \biggl[
        \textsc{KL}\left(
            p_\q(z_{b} \mid x, z_{-b})
            \,||\,
            q_\f(z_{b} \mid x, z_{-b})
        \right)
    \biggr]
    ,
    \quad
    b = 1, 2, ..., B
    .
\end{align*}
We take gradient w.r.t~$\f$ as
\begin{align}
    \label{eq:g-phi-b}
    g_\phi^b 
    :&= - \nabla_\phi \E_{p(x)}\left[ \E_{p_\theta(z_{-b} \given x)}\left[\textsc{kl}\left(p_\theta(z_b \given z_{-b}, x) || q_\phi(z_b \given z_{-b}, x)\right)\right]
    \right]
    \\
    &= 
    \E_{p(x)}\left[
    \E_{p_\theta(z_{1:B} \given x)}\left[\nabla_\phi \log q_\phi(z_b \given z_{-b}, x)\right]
    \right]
    , \quad
    b = 1, 2, ..., B
    . 
\end{align}
We compute a self-normalized gradient estimate (line~\ref{line:apg-grad-phi}) in a propose-weigh-reassign manner (line~\ref{line:apg-propose}, line~\ref{line:apg-weight}, line~\ref{line:apg-reassign}).

We will validate this gradient estimate (line~\ref{line:apg-grad-phi}) using the proper weighting again, i.e. we want to prove that 
\begin{align}
    z_{1:B}^l, w^l \pw p_\theta(z_{1:B}, x), \qquad l = 1, \ldots, L.
    \label{eq:invariant}
\end{align}
so that Equation~\ref{eq:pw-estimation} will guarantee the validity of this gradient estimate.

To prove that one Gibbs sweep (line~\ref{line:apg-sweep-begin} - line~\ref{line:apg-sweep-end}) also preserves proper weighting, we will show that each block update satisfies all the 3 conditions (Equation~\ref{eq:pw-of-p}, ~\ref{eq:w-prime} and ~\ref{eq:invariant}) in Proposition~\ref{proposition:extendedspace}, by which we can conclude the samples are still properly weighted after each block update. 

Without loss of generality, we drop all $l$ superscripts in the rest of the proof. Before any block update (before line~\ref{line:apg-propose}),  we already know that samples are properly weighted, i.e.
\begin{align}
    z, w \pw p_\theta(z, x).
\end{align}
which corresponds to Equation~\ref{eq:pw-of-p}. Next we define a conditional distribution $q(z' \mid z):= q_\phi(z_b' \given x, z_{-b}) \delta_{z_{-b}}(z_{-b}')$, from which we propose a new sample
\begin{align}
    z' \sim q_\phi(z_b' \given x, z_{-b}) \delta_{z_{-b}}(z_{-b}'),
\end{align}
where the density of $z_{-b}'$ is a delta mass on $z_{-b}$ defined as $\delta_{z_{-b}}(z_{-b}') = 1$ if $z_{-b} = z_{-b}'$ and $0$ otherwise.
In fact, this form of sampling step is equivalent to firstly sample $z_b' \sim q_\phi(z_b \mid x, z_{-b})$ (line~\ref{line:apg-propose}) and let $z_{-b}' = z_{-b}$ (line~\ref{line:apg-reassign}), which is exactly what the APG sampler assumes procedurally in Algorithm~\ref{alg:amortized-gibbs}. This condition corresponds to Equation~\ref{eq:z-prime}.

Finally, we define the weight $w'$
\begin{align}
    w' = \frac{{\color{blue}p_\theta(x, z_b', z_{-b}')} {\color{red}r(z_b \given x, z_{-b}) \delta_{z_{-b}}(z_{-b})}}{{\color{red}p_\theta(x, z_b, z_{-b})} {\color{blue}q_\phi(z_b' \given x, z_{-b}) \delta_{z_{-b}}(z_{-b}')}} w,
\end{align}
where the terms in blue are treated as densities (normalized or unnormalized) of $z_{1:B}'$ and the terms in red are treated as densities of $z_{1:B}$.
Since both delta mass densities evaluate to one, this weight is equal to the weight computed after each block update (line~\ref{line:apg-weight}). This condition corresponds to Equation~\ref{eq:w-prime}.

Now we can apply the conclusion \eqref{eq:to-prove} in Proposition~\ref{proposition:extendedspace} and claim 
\begin{align*}
z_{1:B}', w' \pw p_\theta(z_{1:B}', x)
.
\end{align*}
since $z_{-b} = z_{-b}'$ and $z_b = z_b'$ due to the re-assignment (line~\ref{line:apg-reassign}). Note that the resampling step (line~\ref{line:resample}) will preserve the proper weighting because of Proposition~\ref{proposition:resampling}.

Based on the fact that proper weighting is preserved at both the initial proposing step $k=1$ and one Gibbs sweep $k=2$, we have proved that both gradient estimates (line~\ref{line:rws-grad-phi} and line~\ref{line:apg-grad-phi}) are correct.

\section{Architectures of Amortized Population Gibbs samplers using Neural Sufficient Statistics}
\label{appendix:architecture}
Based on our proposed parameterization in terms of neural sufficient statistics (see section~\ref{sec:neural-sufficient-statistics}), we will explain how we design the approximate Gibbs (neural) proposals in the experiments in section~\ref{sec:experiments}.

In general, we consider a structured model $p_\q(x, z)$ where we can partition the latent variables $z = \{z^\textsc{g}, z^\textsc{l}\}$ into global variables $z^\textsc{g}$ and local variables $z^\textsc{l}$. The dimensionality of global variables is typically constant, whereas local variables $\z^\textsc{l} = \{\z^\textsc{l}_1, \ldots, \z^\textsc{l}_N\}$ have a dimensionality that increases with the instance size $N$. For models with this structure, the local variables are typically conditionally independent 
\begin{align}
z^\textsc{L}_n \bot z^\textsc{L}_{-n} \mid x, z^\textsc{g}.    
\end{align}
We assume that the priors $p(z^\textsc{g}; \lambda^\textsc{g})$ and $p(z^\textsc{l}; \lambda^\textsc{l})$ are in the exponential family form, where $\lambda^\textsc{g}$ and $\lambda^\textsc{l}$ are natural parameters of the corresponding distributions. By the conditional independencies, we parameterize the conditional neural proposals (i.e.~variational distributions) using neural sufficient statistics $T_\f(\cdot)$ as
\begin{align}
    \label{eq:appendix-neural-sufficient-statistics}
    \tilde{\lambda}^\textsc{g}
    &= 
    \lambda^\textsc{g} 
    + 
    \sum_{n=1}^N 
    T^\textsc{g}_\f(\x_n, z^\textsc{l}_{n}),
    &
    \tilde{\lambda}^\textsc{l}_n
    &=
    \lambda^\textsc{l}_n + T^\textsc{l}_\f(\x_n, \z^\textsc{g}).
\end{align}
where $\tilde{\lambda}^\textsc{g}$ and $\tilde{\lambda}^\textsc{l}$ are natural parameters of proposals of the global variables $z^\textsc{g}$ and local variables $z^\textsc{l}$ respectively. 

In the Gaussian mixture model we know the analytic forms of conditional (conjugate) posteriors. This means that we have analytic expressions for true sufficient statistics in equation~\ref{eq:appendix-neural-sufficient-statistics}. Here, the APG sampler learns neural sufficient statistics that approximate the true statistics. In the deep generative mixture model and the time series model in bouncing MNIST, we no longer know the analytic forms of the conditionals. As a result, the APG samplers for these two models will employ neural networks $f_\f^\textsc{g}$ and $f_\f^\textsc{l}$ that take the (learned) neural sufficient statistics and the parameters of the priors as input, and predict the parameters of the proposals as output, i.e.
\begin{align*}
    \label{eq:appendix-neural-sufficient-statistics}
    \tilde{\lambda}^\textsc{g}
    &= 
    \lambda^\textsc{g} 
    + 
    \sum_{n=1}^N 
    T^\textsc{g}_\f(\x_n, z^\textsc{l}_{n})
    \approx
    f_\f^\textsc{g}(\lambda^\textsc{g}, \, \sum_{n=1}^N 
    T^\textsc{g}_\f(\x_n, z^\textsc{l}_{n}))
    ,
    &
    \tilde{\lambda}^\textsc{l}_n
    &=
    \lambda^\textsc{l}_n + T^\textsc{l}_\f(\x_n, \z^\textsc{g})
    \approx
    f_\f^\textsc{l}(\lambda^\textsc{l}_n, \,
    T^\textsc{l}_\f(\x_n, z^\textsc{g}))    
    .
\end{align*}
Since there is always a deterministic transformation between a natural parameter and the corresponding distribution parameters (i.e. the parameters in canonical form), we can always convert any exponential family to a canonical form. For convenience, our networks output the canonical parameters directly, rather than returning natural parameters that then need to be converted to canonical form.

\subsection{Gaussian Mixture Model}
In the APG sampler for the Gaussian mixture model (GMM), we employ neural proposals of the form 
\begin{align}
    q_\f(\mu_{1:M}, \tau_{1:M} \mid x_{1:N}) 
    &=
    \prod_{m=1}^M
    \text{NormalGamma}
    \Big(
        \mu_{m}, \tau_{m} 
        \:\Big\vert\:
        \tilde{\alpha}_m,
        \tilde{\beta}_m,
        \tilde{\mu}_m,
        \tilde{\nu}_m  
    \Big),
    \\
    q_\f(\mu_{1:M}, \tau_{1:M} \mid x_{1:N}, c_{1:N}) 
    &=
    \prod_{m=1}^M
    \text{NormalGamma}
    \Big(
        \mu_m, \tau_m 
        \:\Big\vert\:
        \tilde{\alpha}_m,
        \tilde{\beta}_m,
        \tilde{\mu}_m,
        \tilde{\nu}_m   
    \Big)
    ,
    \\
    q_\f(c_{1:N} \mid x_{1:N}, \mu_{1:M}, \tau_{1:M})
    &=
    \prod_{n=1}^N
    \text{Categorical}
    \Big(
    c_n
    \:\Big\vert\:
    \tilde{\pi}_n
    \Big).
\end{align}
where $M=3$ is the number of clusters in a GMM. We use the tilde symbol $\,\tilde{}\,$ to denote the parameters of the conditional neural proposals (i.e. approximate Gibbs proposals). The NormalGamma on the vector-valued mean $\mu_{m} \in\mathbb{R}^2$ and diagonal precision $\tau_{m} \in\mathbb{R}^2_{+}$ follows the standard definition
\begin{align}
    \tau_m &\sim \text{Gamma}(\alpha_0, \beta_0), 
    &
    \mu_m &\sim \text{Normal}(\mu_0, 1 / (\nu_0 \tau)).
\end{align}
where $\mu_0 = 0, \nu_0 = 0.1, \alpha_0 = 0.2, \beta_0 = 0.2$. The natural parameters $\lambda^\textsc{g} := (\lambda_1^\textsc{g}, \lambda_2^\textsc{g}, \lambda_3^\textsc{g}, \lambda_4^\textsc{g})$ of this distribution are defined in terms of the canonical parameters as 
\begin{align}
    \lambda_1^\textsc{g} &= \alpha_0 - \frac{1}{2},
    &
    \lambda_2^\textsc{g} &= - \beta_0 - \frac{\nu_0 \mu_0^2}{2}, 
    &
    \lambda_3^\textsc{g} &= \nu_0 \mu_0,
    &
    \lambda_4^\textsc{g} &= - \frac{\nu_0}{2}.
\end{align}
We employ neural sufficient statistics that approximate the true pointwise sufficient statistics
\begin{align*}
    \Big\{\mathrm{I}[c_n \!=\! m], 
        ~\mathrm{I}[c_n \!=\! m] \, x_n, 
        ~\mathrm{I}[c_n \!=\! m] \, x_n^2 
        ~\Big\vert~m \!=\! 1,2,\dots,M 
    \Big\}
\end{align*}

We use fully-connected networks for the statistics $T^\textsc{g}_\phi(x_n)$ of the initial proposal $q_\f(\mu_{1:M}, \tau_{1:M} \mid x_{1:N})$ and the statistics $T^\textsc{g}_\phi(x_n,c_n)$ for the conditional $q_\f(\mu_{1:M}, \tau_{1:M} \mid x_{1:N}, c_{1:N})$,
\begin{center}
    \begin{tabular}{c|c}
    \toprule
    \multicolumn{2}{c}{$T^\textsc{g}_\phi(x_n), \: x_n\in\mathbb{R}^2$}
    \\
    \midrule
    \parbox{3cm}{\centering FC. 2. $(s_n)$}
    & \parbox{3cm}{\centering FC. 3. Softmax. $(t_n)$}
    \\
    \bottomrule
    \end{tabular}
    \hspace{4em}
    \begin{tabular}{c|c}
    \toprule
    \multicolumn{2}{c}{$T^\textsc{g}_\phi(x_n,c_n), \: x_n\in\mathbb{R}^2, c_n\in\{0, 1\}^3$}
    \\
    \midrule
    \multicolumn{2}{c}{Concatenate$[x_n, \: c_n]$}
    \\
    \hline
     \parbox{3cm}{\centering FC. 2. $(s_n)$}
    & \parbox{3cm}{\centering FC. 3. Softmax. $(t_n)$}
    \\
    \bottomrule
    \end{tabular}
    \label{arch-gmm-global}
\end{center}
The output of each network consists of two elements $s_n$ and $t_n$. $s_n$ approximates the variable $x_n$; $t_{n, m}$ approximates the identity function $\mathrm{I}[c_n \!=\! m]$. 
Then we sum over all the points and compute the parameters of the conjugate posterior in analytic forms

\begin{align}
    N_m &= \sum_{n=1}^N t_{n, m},
    \hspace{1em}
    \tilde{x}_m = \sum_{n=1}^N t_{n, m} \cdot s_n,
    \hspace{1em}
    \tilde{x}_m^2 = \sum_{n=1}^N t_{n, m} \cdot s_n^2,
    \label{eq:appendix-gmm-aggregation-nss}
    \\
    \tilde{\alpha}_m &= \alpha_0 + \frac{N_m}{2},
    \\
    \tilde{\beta}_m &= \beta_0 + \frac{1}{2} \left(
    \tilde{x}_m^2 
    -  
    \frac{(\tilde{x}_m)^2}{N_m} 
    \right)
    +
    \frac{N_m \nu_0}{N_m + \nu_0} 
    \frac{(\frac{\tilde{x}_m}{N_m} - \mu_0)^2}{2}
    ,
    \\
    \tilde{\mu}_m &= \frac{\mu_0 \nu_0 + \tilde{x}_m}{\nu_0 + N_m},
    \\
    \tilde{\nu}_m &= \nu_0 + N_m
    .
\end{align}

We assume a Categorical prior on the assignment of each point $c_{n}$ of the form
\begin{align}
    c_n \sim \text{Categorical}(\pi)
\end{align}
where $\pi = (\frac{1}{3}, \frac{1}{3}, \frac{1}{3})$. The natural parameter is
\begin{align}
    \lambda^\textsc{l} = \log \pi
\end{align}
We employ a vector of neural statistics
\begin{align}
T^\textsc{l}_\phi(x_n, \mu_{1:M}, \tau_{1:M}) 
:=  
\big(
    T^\textsc{l}_\phi(x_n, \mu_1, \tau_1), 
    \:
    T^\textsc{l}_\phi(x_n, \mu_2, \tau_2), 
    \:
    \dots, 
    \:
    T^\textsc{l}_\phi(x_n, \mu_M, \tau_M)
\big),
\end{align}
where each element is parameterized by the network
\begin{table}[!h]
    \centering
    \begin{tabular}{c}
    \toprule
        $T^\textsc{l}_\phi(x_n, \mu_m, \tau_m), \: x_n\in\mathbb{R}^2, \: \mu_m\in\mathbb{R}^2, \: \tau_m\in\mathbb{R}^2_{+}$
        \\
    \midrule
    Concatenate$[x_n\, \: \mu_m, \: \tau_m]$\\
    \hline
    \parbox{4cm}{\centering FC. 32. Tanh. FC. 1.}\\
    \bottomrule
    \end{tabular}
    \label{arch-gmm-local}
\end{table}

We add these statistics to the natural parameters. The canonical parameters of the approximate posterior $\tilde{\pi}$ are then simply the Softmax normalization of the resulting sum
\begin{align}
    \tilde{\pi}_n
    =
    \text{Softmax}
    \Big(
    \log \pi + T^\textsc{l}_{\phi}(x_n, \, \mu_{1:M}, \, \tau_{1:M})
    \Big).
\end{align}

\subsection{Deep Generative Mixture Model}
The APG sampler in the deep generative mixture model (DMM) employs neural proposals of the form
\begin{align}
    q_\f(\mu_{1:M} \mid x_{1:N}) &= \prod_{m=1}^M \text{Normal} \Big( \mu_m \:\Big\vert\: \tilde{\mu}_m, \tilde{\sigma}^2_m I \Big)
    ,
    \\
    q_\f(\mu_{1:M} \mid x_{1:N}, c_{1:N}, h_{1:N}) &= \prod_{m=1}^M \text{Normal} \Big( \mu_m \:\Big\vert\: \tilde{\mu}_m, \tilde{\sigma}_m^2 I \Big)
    ,
    \\
    q_\f(c_{1:N}, h_{1:N} \mid x_{1:N}, \mu_{1:M}) &= q_\f(c_{1:N} \mid x_{1:N}, \mu_{1:M}) \: q_\f(h_{1:N} \mid c_{1:N}, , x_{1:N}, \mu_{1:M}) 
    \\
    &= \prod_{n=1}^N \text{Categorical} \Big(c_n \:\Big\vert\: \tilde{\pi}_n \Big) \: \text{Beta} \Big( h_n \:\Big\vert\: \tilde{\alpha}_n, \tilde{\beta}_n \Big)
    .
\end{align}
where $M=4$ is the number of clusters in a GMM. We use the tilde symbol $\,\tilde{}\,$ to denote the parameters of the conditional neural proposals (i.e. approximate Gibbs proposals).

\newpage
We assume a Gaussian prior on mean of each cluster $\mu_m$ of the form
\begin{align}
    \mu_m \sim \text{Normal} (\mu, \sigma^2_0 I)
\end{align}
where the hyper-parameters are $\mu = 0$ and $\sigma_0 = 10$. To compute the parameters of neural proposals, we firstly employ some MLPs that predict neural sufficient statistics $T_\f^\textsc{g}(x_n)$ and $T_\f^\textsc{g}(x_n, c_n, h_n)$ as

\begin{table}[!h]
    \centering
    \begin{tabular}{c|c}
    \toprule
        \multicolumn{2}{c}{$T_\f^\textsc{g}(x_n), \: x_n\in\mathbb{R}^2$}  \\
    \midrule
     \parbox{6cm}{\centering FC. 32. Tanh. FC. 8. $(s_n)$}
    & 
    \parbox{6cm}{\centering FC. 32. Tanh. FC. 4. Softmax. $(t_n)$}\\
    \bottomrule
    \end{tabular}
    \label{arch-dgmm-rws}
\end{table}

\begin{table}[!h]
    \centering
    \begin{tabular}{c|c}
    \toprule
        \multicolumn{2}{c}{$T^\textsc{g}_\f(x_n, c_n, h_n), \: x_n\in\mathbb{R}^2, \, c_n\in\{0, 1\}^4, \, h_n\in[0, 1]^1$} \\
    \midrule
    \multicolumn{2}{c}{Concatenate$[x_n, \: c_n, \: h_n]$} \\
    \hline 
     \parbox{6cm}{\centering FC. 32. Tanh. FC. 8. $(s_n)$}
    & 
    \parbox{6cm}{\centering FC. 32. Tanh. FC. 4. Softmax. $(t_n)$}\\
    \bottomrule
    \end{tabular}
    \label{arch-dgmm-global}
\end{table}

The architectures here are similar to those in the GMM in the sense that output of each network also consists of two features: $s_n$ and $t_n$. The intuition is that the we can extract useful features in the same way, but without conjugacy relationship. Then we compute the outer product of these two features, resulting in a weighted average like the ones in equation~\ref{eq:appendix-gmm-aggregation-nss} as
\begin{align}
    s_n \tens{} t_n := \begin{bmatrix}
                        s_{n, 1} t_{n, 1} & s_{n, 1} t_{n, 2} & s_{n, 1} t_{n, 3} & s_{n, 1} t_{n, 4}\\
                        s_{n, 2} t_{n, 1} & s_{n, 2} t_{n, 2} & s_{n, 2} t_{n, 3} & s_{n, 2} t_{n, 4}\\
                        \cdots \\
                        s_{n, 8} t_{n, 1} & s_{n, 8} t_{n, 2} & s_{n, 8} t_{n, 3} & s_{n, 8} t_{n, 4}
                        \end{bmatrix}
\end{align} 
We aggregate this outer products over all the points by taking an elementwise sum $\sum_{n=1}^N s_n \tens{} t_n$. Then we normalize the aggregation by taking an elementwise division with the sum $\sum_{n=1}^N t_n$, i.e.
\begin{align}
    \tilde{T}_\f^\textsc{g} := \frac{\sum_{n=1}^N s_n \tens{} t_n} {\sum_{n=1}^N t_n}
\end{align}
where we call the normalized aggregation $\tilde{T}_\f^\textsc{g} \in \mathbb{R}^{8\times 4}$, where its second dimension corresponds to the number of cluster $M=4$. Next we employ MLPs $f_\f^\textsc{g}(\cdot)$ to predict variational distribution of each cluster given the aggregated neural sufficient statistics and the parameters of the priors as
\begin{align}
    &
    \tilde{\mu}_m, \tilde{\sigma}_m^2 \xleftarrow{} f_\f^\textsc{g} \Big( \tilde{T}_\f^\textsc{g}(x_n)[:, m], \: \mu, \: \sigma_0 \Big),
    &
    \tilde{\mu}_m, \tilde{\sigma}_m^{2} \xleftarrow{} f_\f^\textsc{g} \Big( \tilde{T}_\f^\textsc{g}(x_n, c_n, h_n)[:, m], \: \mu, \: \sigma_0 \Big).
\end{align}
where $m=1, 2, ..., M$. The notation $[:, m]$ mean we take the m-th slice along the second dimension which corresponds the m-th cluster.

We parameterize the $f_\f^\textsc{g}(\cdot)$ using two separate MLPs, each of which is concatenated with the corresponding pointwise neural sufficient statistics networks, i.e. $T_\f^\textsc{g}(x_n)$ and $T_\f^\textsc{g}(x_n, c_n, h_n)$ respectively
\begin{center}
    \begin{tabular}{c}
    \toprule
    $f_\f^\textsc{g} \Big( \tilde{T}^\textsc{g}_\f(x_n)[:, m], \: \mu, \: \sigma_0 \Big), \: \tilde{T}^\textsc{g}_\f(x_n)[:, m] \in\mathbb{R}^8, \,  \mu \in \mathbb{R}^2, \, \sigma_0 \in\mathbb{R}^2$ \\
    \midrule
    Concatenate$[\tilde{T}^\textsc{g}_\f(x_n)[:, m], \: \mu, \: \sigma_0]$\\
    \hline
    FC. $2\times 32$. Tanh. FC. $2\times 2$\\
    \bottomrule
    \end{tabular}
    \hspace{2em}
    \begin{tabular}{c}
    \toprule
    $f_\f^\textsc{g} \Big( \tilde{T}^\textsc{g}_\f(x_n, c_n, h_n)[:, m], \: \mu, \: \sigma_0 \Big), \: \tilde{T}^\textsc{g}_\f(x_n, c_n , h_n)[:, m] \in\mathbb{R}^8, \: \mu \in \mathbb{R}^2, \: \sigma_0 \in\mathbb{R}^2$ \\
    \midrule
    Concatenate$[\tilde{T}^\textsc{g}_\f(x_n, c_n , h_n)[:, m], \: \mu, \: \sigma_0]$\\
    \hline
    FC. $2\times 32$. Tanh. FC. $2\times 2$\\
    \bottomrule
    \end{tabular}
\end{center}
We assume a Categorical on mixture assignments $c_{n}$ and a Beta prior on the embedding of each point $h_n$,
\begin{align}
    &
    c_n \sim \text{Categorical}(\pi),
    &
    h_n \sim \text{Beta}(\alpha, \beta).
\end{align}
where $\pi = (\frac{1}{4}, \frac{1}{4},  \frac{1}{4}, \frac{1}{4}), \alpha = 1$ and $\beta = 1$. 

The neural sufficient statistics is defined as 
\begin{align}
T^\textsc{l}_\phi(x_n, \mu_{1:M}) :=  \big( T^\textsc{l}_\phi(x_n, \mu_1), \: T^\textsc{l}_\phi(x_n, \mu_2), \: ..., \: T^\textsc{l}_\phi(x_n, \mu_M) \big)
\end{align}
where each element is parameterized by the network
\begin{table}[!h]
    \centering
    \begin{tabular}{c}
    \toprule
        $T^\textsc{l}_\phi(x_n, \mu_m), \: x_n\in\mathbb{R}^2, \: \mu_m\in\mathbb{R}^2$
        \\
    \midrule
    Concatenate$[x_n, \: \mu_m]$\\
    \hline
    \parbox{4cm}{\centering FC. 32. Tanh. FC. 1.}\\
    \bottomrule
    \end{tabular}
    \label{arch-gmm-local}
\end{table}
  

We compute the parameters of the conditional proposal $\tilde{\pi}_n$ by computing the logits normalizing it using Softmax activation
\begin{align}
    \tilde{\pi}_n
    =
    \text{Softmax}
    \Big(
    \log \pi + T^\textsc{l}_{\phi}(x_n, \, \mu_{1:M})
    \Big).
\end{align}
Then we can sample assignments $c_n$ from the variational distribution $c_n \sim \text{Categorical}(\tilde{\pi}_n)$. The neural proposal for embedding variable $h_n$ is conditioned on the assignments $c_n$ in a way that it takes as input the mean of the cluster $\mu_m$, to which each point belongs to, i.e. 
\begin{align}
    q_\phi(h_n \mid x_n, \mu_{1:M}, c_n) &= q_\phi(h_n \mid x_n, \mu_{c_n}).
\end{align}
As a result, the network for this neural proposal is
\begin{table}[!h]
    \centering
    \begin{tabular}{c}
    \toprule
        $q_\phi(h_n \mid x_n, \mu_{c_n}), \: x_n \in\mathbb{R}^2, \: \mu_{c_n} \in\mathbb{R}^2$
        \\
    \midrule
    $x_n - \mu_{c_n}$\\
    \hline
    \parbox{4cm}{\centering FC. $2\times 32$. Tanh. FC. $2\times 1$.}\\
    \bottomrule
    \end{tabular}
    \label{arch-gmm-local}
\end{table}

Since the Beta distribution requires its parameters to be positive, this network will output the logarithms of proposal parameters for $h_n$ of the form
\begin{align}
    \log \tilde{\alpha}_n, \: \log \tilde{\beta}_n \xleftarrow{} q_\phi(h_n \mid x_n, \mu_{c_n}).
\end{align}
In this experiment we also learn a deep generative model of the form
\begin{align}
    p_\q(x_{1:N} | \mu_{1:M}, c_{1:N}, h_{1:N}) := \prod_{n=1}^N \text{Normal}
    \Big(x_n  \:\Big\vert\:  g_\q(h_n) + \mu_{c_n}, \sigma_\epsilon^2 I \Big).
\end{align}
where $\sigma_\epsilon=0.1$ is a hyper-parameter. The architecture of the MLP decoder $g_\q(h_n)$ is
\begin{table}[!h]
    \centering
    \begin{tabular}{c}
    \toprule
    $g_\q(h_n), \: h_n \in\mathbb{R}^1$ \\
    \midrule
    FC. 32. Tanh. FC. 2. \\
    \bottomrule
    \end{tabular}
    \label{arch-dgmm-decoder}
\end{table}

\newpage
\subsection{Time Series Model -- Bouncing MNIST}
We learn a deep generative model of the form
\begin{align}
    p_\q(x_{1:T} \mid z^{\mathrm{what}}_{1:D}, z^{\mathrm{where}}_{1:T})
    =
    \prod_{t=1}^T
    \text{Bernoulli}
    \Big(
    x_t \:\Big\vert\:
        \sigma
        \Big(
            \sum_{d} \mathrm{ST}
            \big(
                g_\q (\z_{d}^{\mathrm{what}})
                , 
                \:
                z^{\mathrm{where}}_{d, t}
            \big)
        \Big)
    \Big)
\end{align}
Given each digit feature $z^{\mathrm{what}}_{d}$, the APG sampler reconstruct a $28\times28$ MNIST image using a MLP decoder, the architecture of which is

\begin{table}[!h]
\centering
\label{arch-bmnist-decoder}
\begin{tabular}{l}
    \toprule
     $g_\q(z^{\mathrm{what}}_d), \: z^{\mathrm{what}}_d \in\mathbb{R}^{10}$ \\
    \midrule
    FC. 200. ReLU. \\
    \hline
    FC. 400. ReLU.\\
    \hline
    FC. 784. Sigmoid. \\
    \bottomrule
\end{tabular}
\end{table}

Then we put each reconstructed image $g_\q(z^{\mathrm{what}}_d)$ onto a $96\times96$ canvas using a spatial transformer ST which takes position variable $z^{\mathrm{where}}_{d, t}$ as input. To ensure a pixel-wise Bernoulli likelihood, we clip on the composition as
\begin{align}
    \text{For each pixel} \, p_i \in \Big( \sum_{d} \mathrm{ST} \big( g_\q (\z_{d}^{\mathrm{what}}), \:z^{\mathrm{where}}_{d, t} \big) \Big),
    \:
    \sigma(p_i)
    = \begin{cases}
        p_i = 0 & \text{if} \, p_i < 0 \\    
        p_i = p_i & \text{if} \, 0 \leq p_i \leq 1 \\   
        p_i = 1 & \text{if} \, p_i > 1 \\
      \end{cases}
\end{align}

The APG sampler in the bouncing MNIST employs neural proposals of the form
\begin{align}
    q_\f (z^{\mathrm{where}}_{1:D, t} \mid x_{t})
    &=
    \prod_{d=1}^D \text{Normal} \Big(z^{\mathrm{where}}_{d, t}  \:\Big\vert\: \tilde{\mu}_{d, t}^\mathrm{where}, \tilde{\sigma}^{\mathrm{where} \, 2}_{d, t} I \Big),
    \qquad \text{for} \: t = 1, 2, \dots, T,
    \\
    q_\f (z^{\mathrm{where}}_{1:D, t} \mid x_{t}, z^{\mathrm{what}}_{1:D})
    &=
    \prod_{d=1}^D \text{Normal} \Big(z^{\mathrm{where}}_{d, t}  \:\Big\vert\: \tilde{\mu}_{d, t}^\mathrm{where}, \tilde{\sigma}^{\mathrm{where} \, 2}_{d, t} I \Big),
    \qquad \text{for} \: t = 1, 2, \dots, T,
    \\
    q_\f (z^{\mathrm{what}}_{1:D} \mid x_{1:T}, z^{\mathrm{where}}_{1:T})
    &=
    \prod_{d=1}^D \text{Normal} \Big(z^{\mathrm{what}}_{d}  \:\Big\vert\: \tilde{\mu}_d^\mathrm{what}, \tilde{\sigma}^{\mathrm{what} \, 2}_d I \Big)
    .
\end{align}
We train the proposals with instances containing $D=3$ digits and $T=10$ time steps and test them with instances containing up to $D=5$ digits and $T=100$ time steps. We use the tilde symbol $\,\tilde{}\,$ to denote the parameters of the conditional neural proposals (i.e. approximate Gibbs proposals).

The APG sampler uses these proposals to iterate over the $T+1$ blocks 
\begin{align*}
    \{z_{1:D}^{\mathrm{what}}\}, 
    \: 
    \{z_{1:D, 1}^{\mathrm{where}}\}, 
    \: 
    \{z_{1:D, 2}^{\mathrm{where}}\}, 
    \:
    \dots, 
    \: 
    \{z_{1:D, T}^{\mathrm{where}}\}.
\end{align*}

For the position features, the proposal $q_\f (z^{\mathrm{where}}_{1:D, t} \mid x_{t})$ and proposal $q_\f (z^{\mathrm{where}}_{1:D, t} \mid x_{t}, z^{\mathrm{what}}_{1:D})$ share the same network, but contain different pre-steps where we compute the input of that network. The initial proposal $q_\f (z^{\mathrm{where}}_{1:D, t} \mid x_{t})$ will convolve the frame $x_t$ with the mean image of the MNIST dataset; The conditional proposal $q_\f (z^{\mathrm{where}}_{1:D, t} \mid x_{t}, z^{\mathrm{what}}_{1:D})$ will convolve the frame $x_t$ with each reconstructed MNIST image $g_\q(z^\mathrm{what}_d)$. We perform convolution sequentially by looping over all digits $d=1, 2, ..., D$. Here is pseudocode of both pre-steps:

\begin{algorithm}[!h]
    \setstretch{1.2}
  \caption{Convolution Processing for  $q_\f (z^{\mathrm{where}}_{1:D, t} \mid x_{t})$}
\begin{algorithmic}[1]
  \State \textbf{Input} frame $x_t\in\mathbb{R}^{9216}$, mean image of MNIST dataset $mm \in \mathbb{R}^{784}$
  \For {$d = 1$ \textbf{to} $D$}
    \State $x_{d, t}^{\text{conv}} \xleftarrow{} \text{Conv2d}(x_t)$ with kernel $mm$, stride = 1, no padding.
    \EndFor
  \State \textbf{Output} Convolved features $\{x_{d, t}^{\text{conv}} \in\mathbb{R}^{4761}\}_{d=1}^D$
\end{algorithmic}
\end{algorithm}

\begin{algorithm}[!h]
    \setstretch{1.2}
  \caption{Convolution Processing for  $q_\f (z^{\mathrm{where}}_{1:D, t} \mid x_{t}, z^{\mathrm{what}}_{1:D})$}
\begin{algorithmic}[1]
  \State \textbf{Input} frame $x_t\in\mathbb{R}^{9216}$, reconstructed MNIST digits $\{g_\q(z^\mathrm{what}_d) \in \mathbb{R}^{784}\}_{d=1}^D$
  \For {$d = 1$ \textbf{to} $D$}
    \State $x_{d, t}^{\text{conv}} \xleftarrow{} \text{Conv2d}(x_t)$ with kernel $g_\q(z^\mathrm{what}_d)$, stride = 1, no padding.
    \EndFor
  \State \textbf{Output} Convolved features $\{x_{d, t}^{\text{conv}} \in\mathbb{R}^{4761}\}_{d=1}^D$
\end{algorithmic}
\end{algorithm}

\newpage
We employ a MLP encoder $f_\f^\textsc{l}(\cdot)$ that takes the convolved features as input and predict the variational parameters for positions $\{z^\mathrm{where}_{d, t}\}_{d=1}^D$ at step $t$, i.e. vector-valued mean $\tilde{\mu}^\mathrm{where}_{d, t}$ and logarithm of the diagonal covariance $\log \tilde{\sigma}^{\mathrm{where} \, 2}_{d, t}$ as
\begin{align}
    &   
    \tilde{\mu}^\mathrm{where}_{d, t}, \log \tilde{\sigma}^{\mathrm{where} \, 2}_{d, t} \xleftarrow{} f_\f^\textsc{l}(x_{d, t}^{\text{conv}})
    ,
    &
    d = 1, 2, \dots, D.
\end{align}
The architecture of the MLP encoder $f_\f^\textsc{l}(\cdot)$ is
\begin{table}[!h]
    \centering
    \begin{tabular}{l}
     \toprule
    $f_\f^\textsc{l}(x_{d, t}^{\text{conv}}), \: x_{d, t}^{\text{conv}} \in \mathbb{R}^{4761}$ \\
    \midrule
    FC. 200. ReLU. \\
    \hline
    FC. $2\times100$. ReLU. \\
    \hline
    FC. $2\times2$. \\
    \bottomrule
    \end{tabular}
    \label{arch-bmnist-enc-where}
\end{table}

For the digit features, the APG sampler performs conditional updates in the sense that we crop each frame $x_t$ into a $28\times28$ subframe according to $z^\mathrm{where}_{d, t}$ using the spatial transformer ST as
\begin{align}
    &x^\text{crop}_{d, t} \xleftarrow{} \text{ST}\big(x_t, z^{\mathrm{where}}_{d, t}\big), &d = 1, 2, \dots, D, \quad t = 1, 2, \dots, T
    .
\end{align}
we employ a MLP encoder $T_\f^\textsc{g}(\cdot)$ that takes the cropped subframes as input, and predicts frame-wise neural sufficient statistics, which we will sum up over all the time steps. The architecture of this network is
\begin{table}[!h]
\centering
\begin{tabular}{c}
    \toprule
    $T_\f^\textsc{g}(x^\text{crop}_{d, t}), \: x^\text{crop}_{d, t}\in\mathbb{R}^{784}$ \\
    \midrule
    FC. 400. ReLU. \\
    \hline
    FC. 200. ReLU. \\
    \bottomrule
\end{tabular}
\end{table}

Then we employ another network $f_\f^\textsc{g}(\cdot)$ that takes the sums as input, and predict the variational parameters for digit features $\{z^\mathrm{what}_{d}\}_{d=1}^D$, i.e. the vector-valued means $\{\tilde{\mu}_d^\mathrm{what}\}_{d=1}^D$ and the logarithms of the diagonal covariances $\{\log \tilde{\sigma}_d^{\mathrm{what} \, 2}\}_{d=1}^D$.  The architecture of this network is
\begin{table}[!h]
    \centering
    \begin{tabular}{c}
    \toprule
    $f_\f^\textsc{g} (\sum_{t=1}^T T_\f^\textsc{g}(x^\text{crop}_{d, t})), \: \sum_{t=1}^T T_\f^\textsc{g}(x^\text{crop}_{d, t}) \in \mathbb{R}^{200}$ \\
    \midrule
    FC. $2\times 10$ \\
    \bottomrule
    \end{tabular}
\end{table}

\newpage
\section{Analytical inclusive KL divergence during Training in GMM}
\label{appendix-kl-gmm-training}
In GMM experiment, we train the model with different number of sweeps $K$ under fixed computational budget $K \cdot L = 100$. Figure~\ref{fig:kl-training-gmm} shows that more number of sweeps results in slightly faster convergence.
\begin{figure}[!h]
    \centering
    \includegraphics[width=170mm]{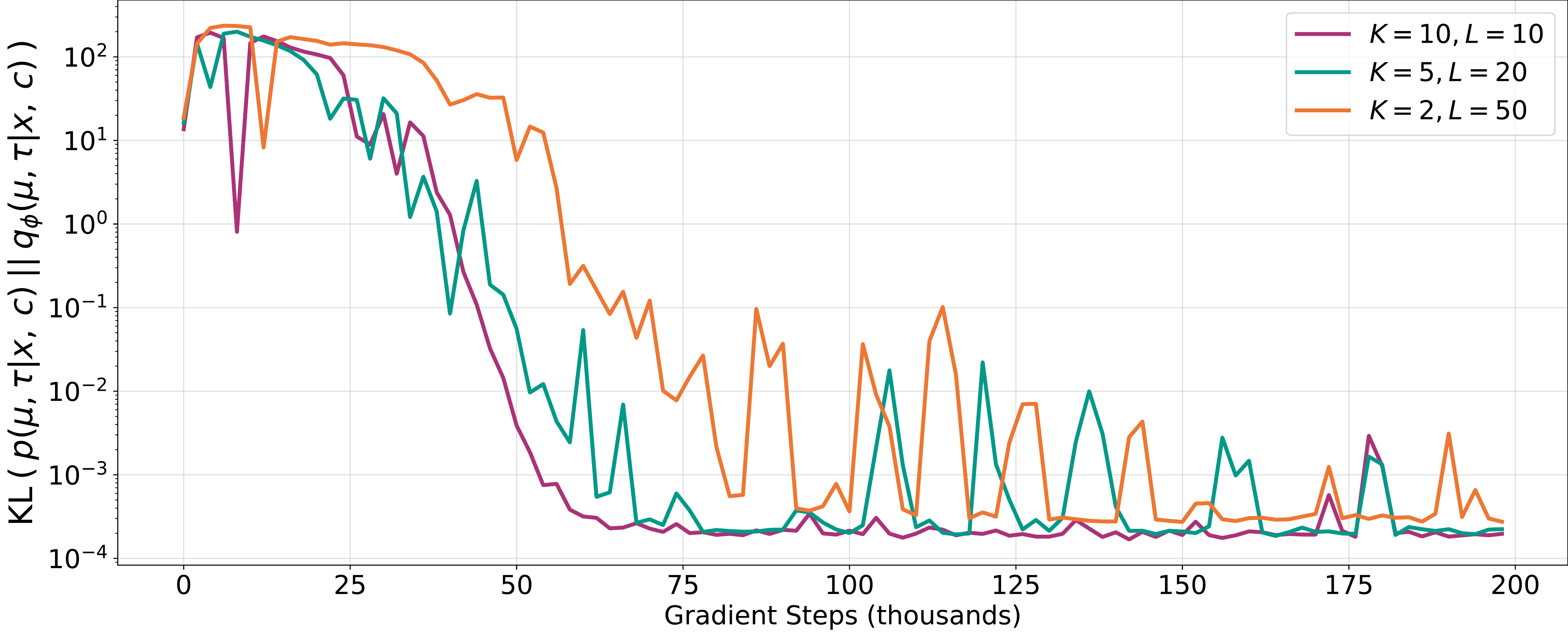}
    \caption{Inclusive KL divergence as a function of gradient steps. Each model is trained with 20000 gradient steps, $K\cdot L=100$}
    \label{fig:kl-training-gmm}
\end{figure}

\section{Comparison between APG sampler and RWS method in Bouncing MNIST}
\label{appendix:bmnist-comparison-rws}
We visualize the inference results and reconstruction from the APG sampler and reweighted wake-sleep method. We can see that APG sampler significantly improves both tracking inference results and the reconstruction on instances with $T=20$ time steps and $D=5$ digits. We can see that the APG sampler achieves substantial results while the RWS method does not make reasonable prediction at all.

\begin{figure*}[!h]
\centering
\includegraphics[width=150mm]{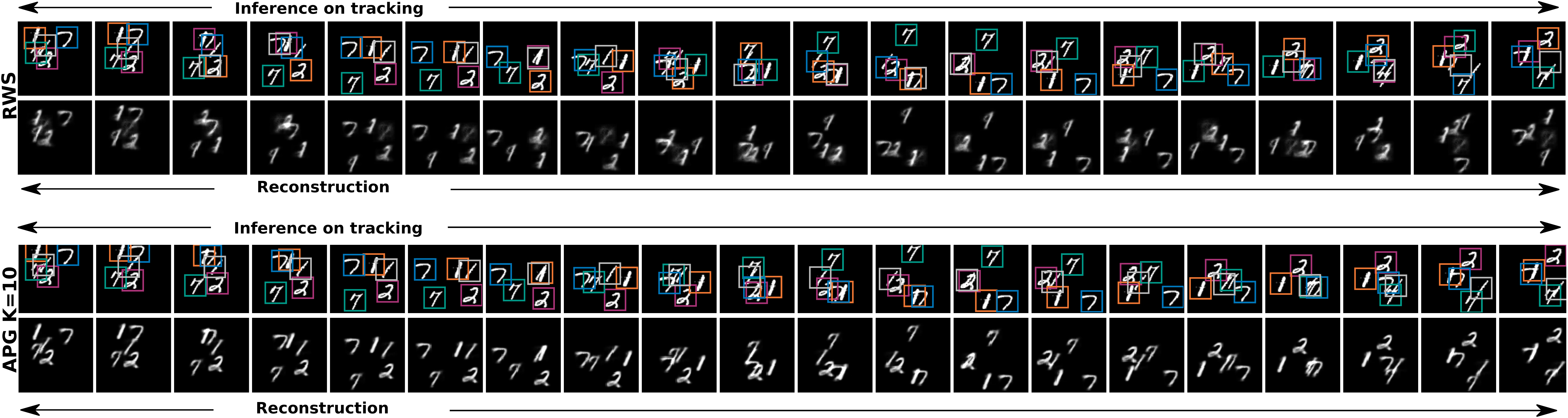}
\caption{Example 1.}
\end{figure*}

\begin{figure*}[!h]
\centering
\includegraphics[width=150mm]{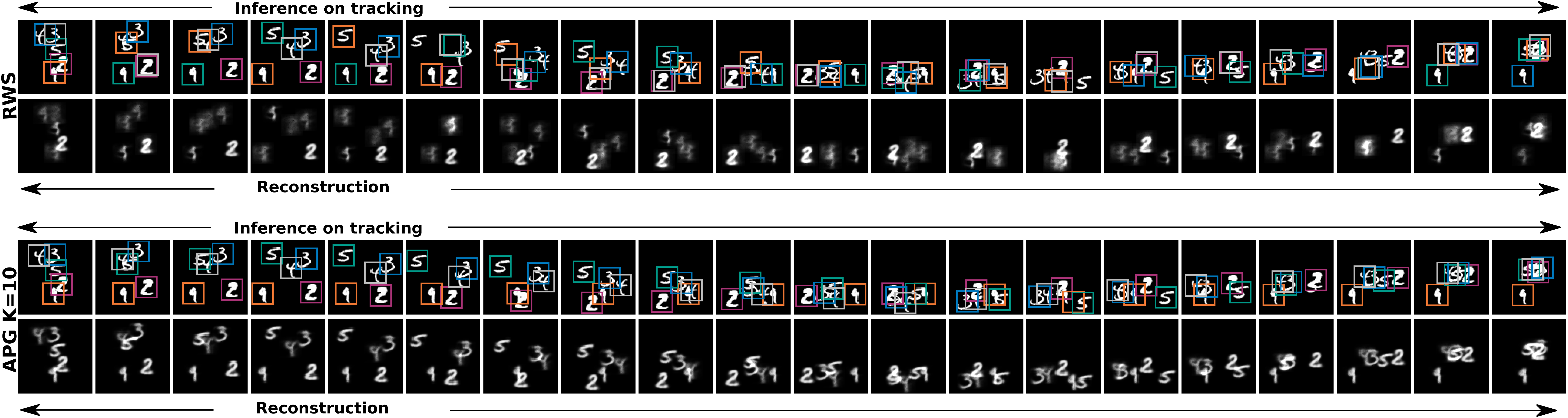}
\caption{Example 2.}
\end{figure*}

\begin{figure*}[!h]
\centering
\includegraphics[width=150mm]{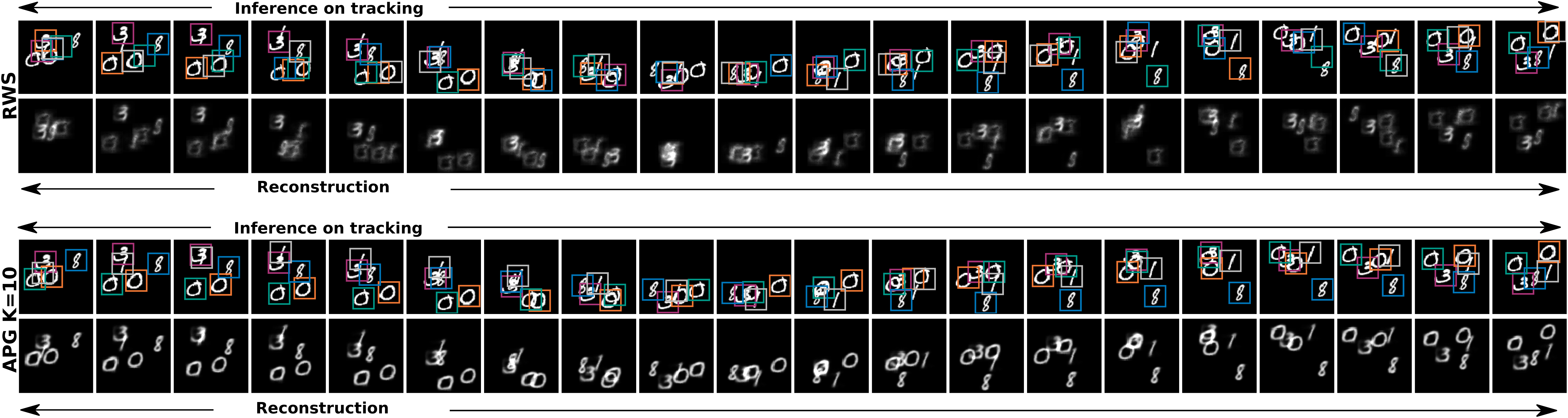}
\caption{Example 3.}
\end{figure*}

\section{Inference Results and Reconstruction on large time steps Bouncing MNIST}
\label{appendix:full-recons}
We show the inference results on tracking and the reconstruction on test instances with $T=100$ time steps and $D=3, 4, 5$ MNIST digits, using models that is trained with instances containing only $T=10$ time steps and $D=3$ digits. In each figure below, the 1st, 3rd, 5th, 7th, 9th rows show the inference results, while the other rows show the reconstruction of the series above. We can see the APG sampler is scalable with much large number of latent variables, achieving accurate inference and making impressive reconstruction.
\begin{figure*}[!h]
\includegraphics[width=170mm]{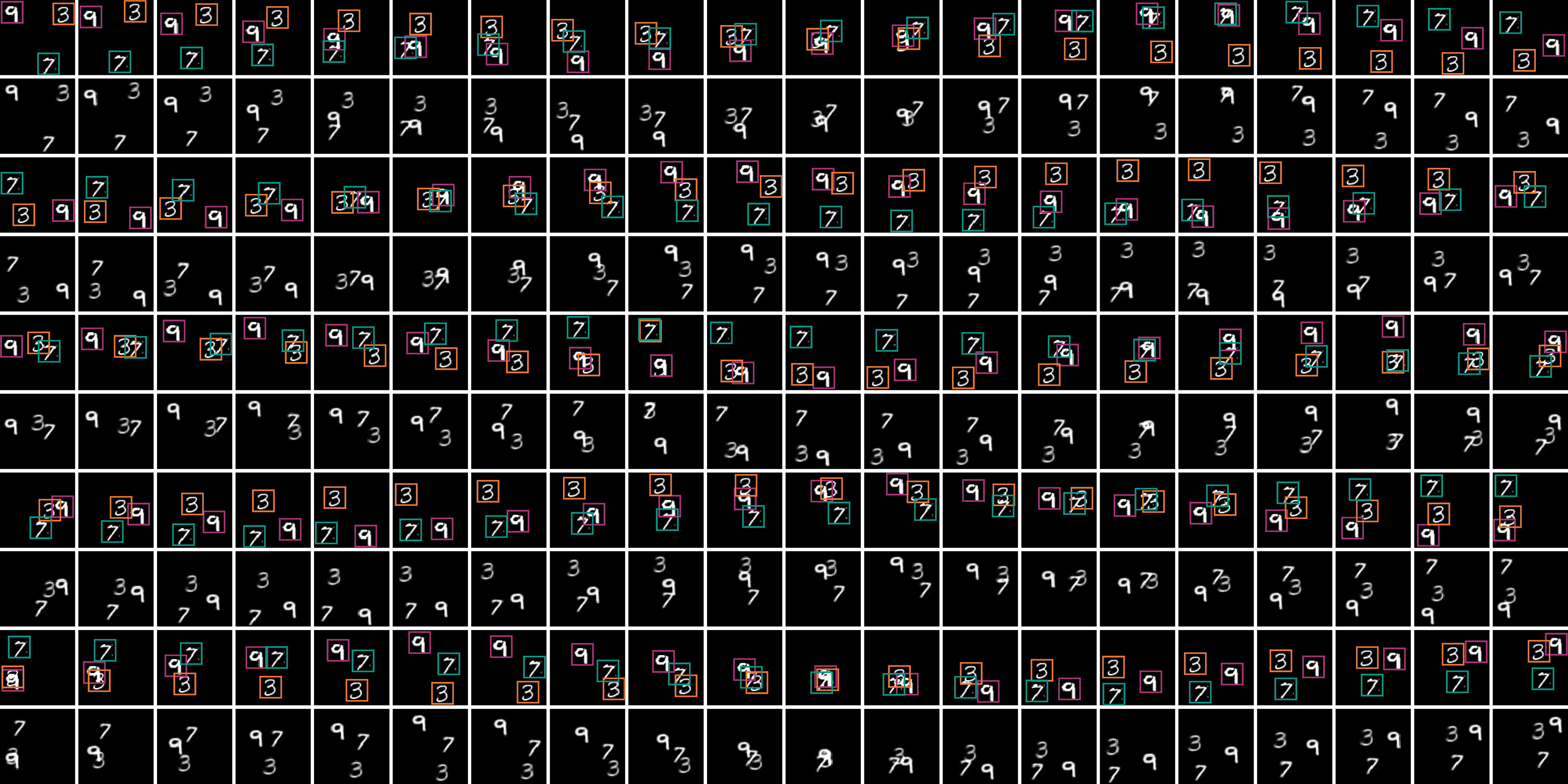}
\caption{Full reconstruction for a video where $T=100, D=3$.}
\end{figure*}

\begin{figure*}[!h]
\includegraphics[width=170mm]{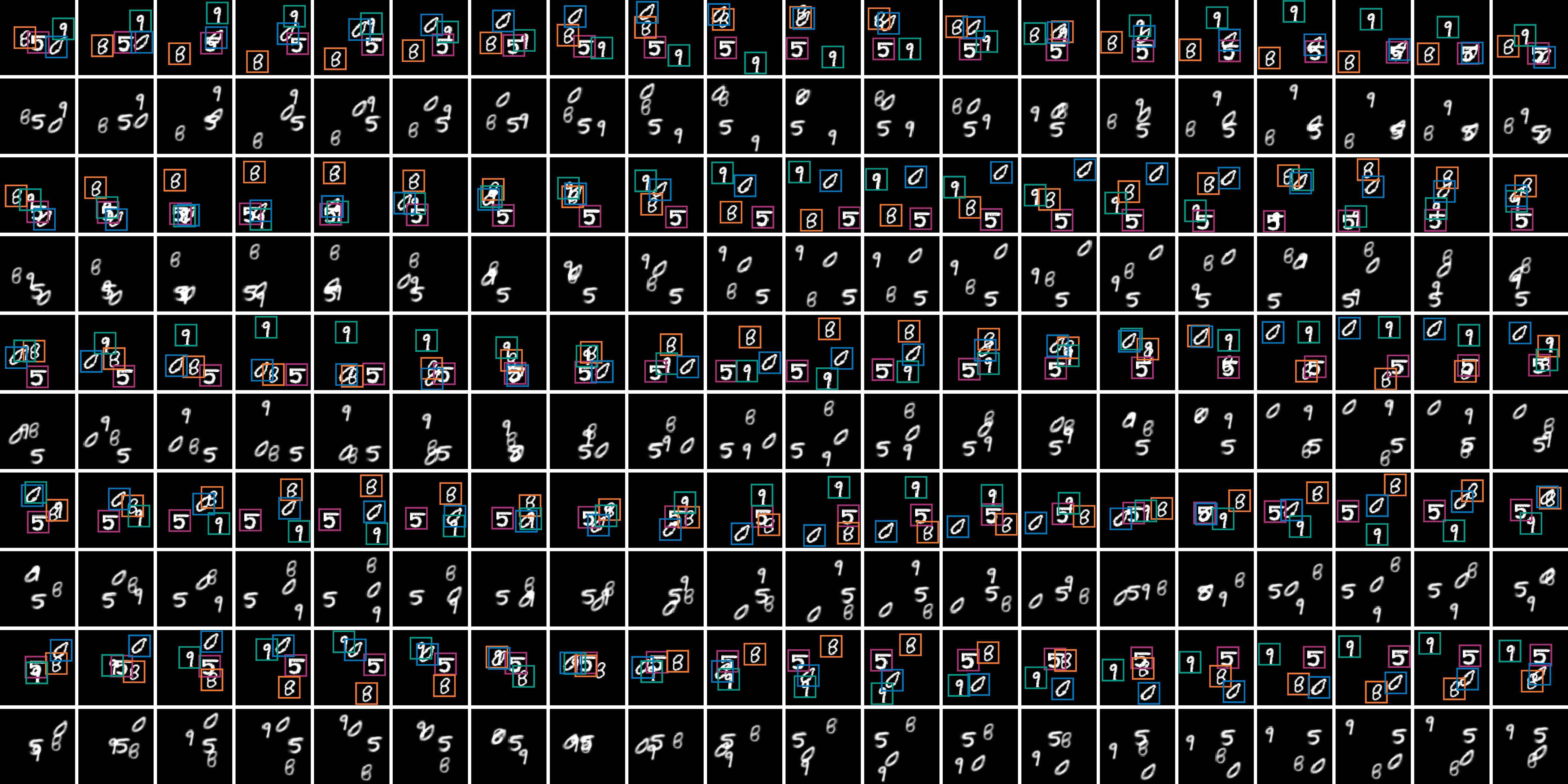}
\caption{Full reconstruction for a video where $T=100, D=4$.}
\end{figure*}
\newpage
\begin{figure*}[!h]
\includegraphics[width=170mm]{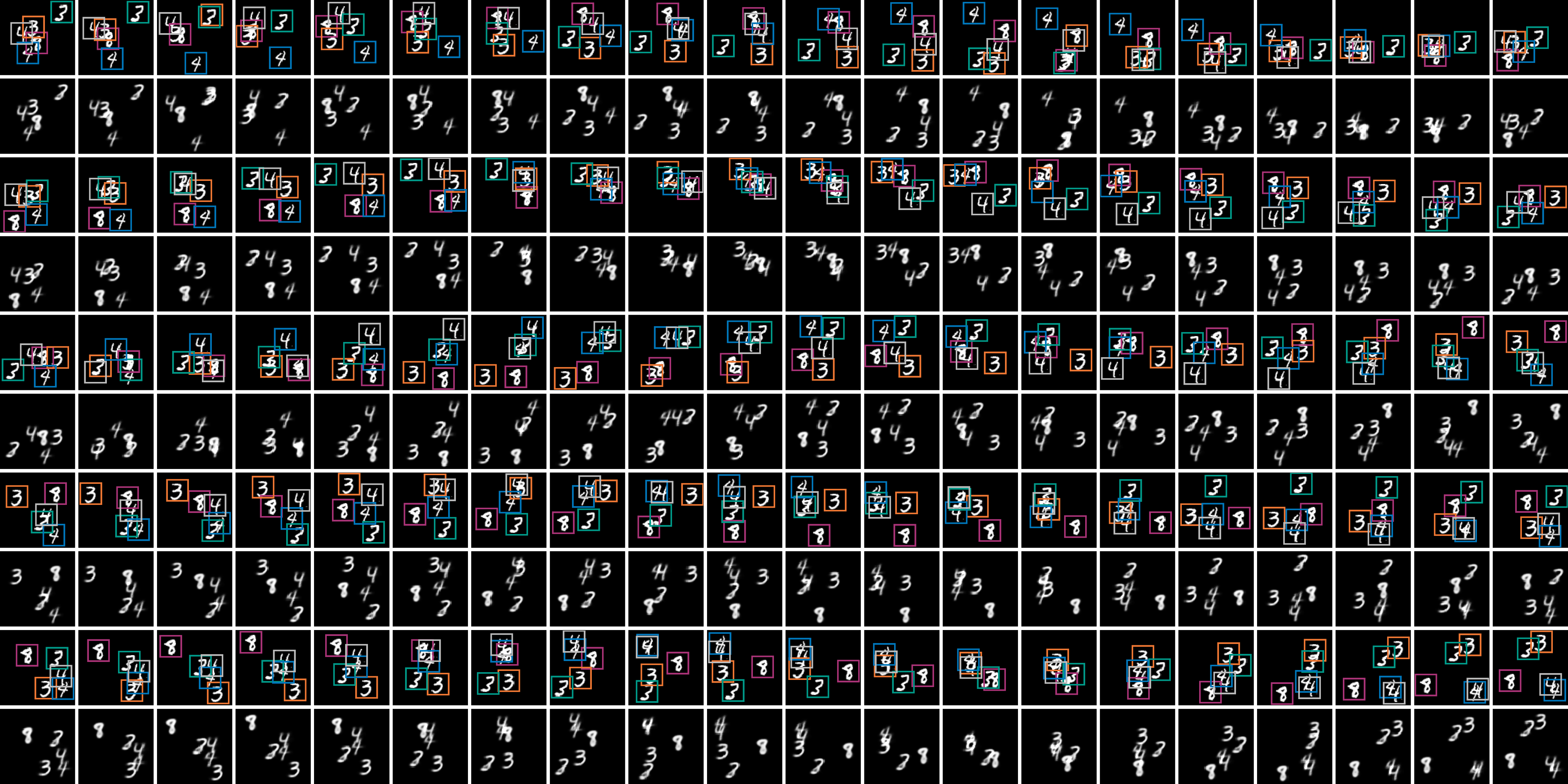}
\caption{Full reconstruction for a video where $T=100, D=5$.}
\end{figure*}

\begin{figure*}[!h]
\includegraphics[width=170mm]{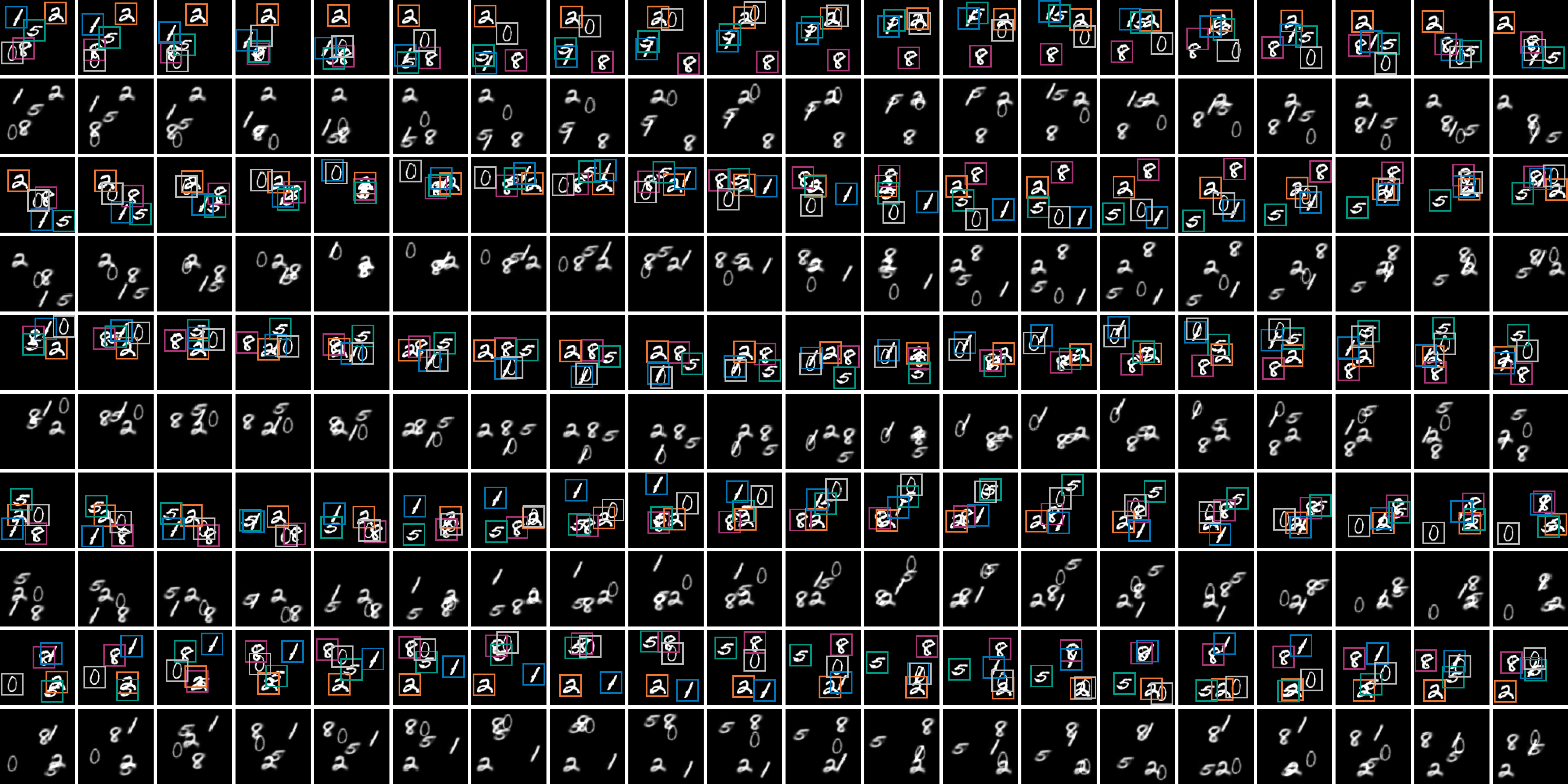}
\caption{Full reconstruction for a video where $T=100, D=5$.}
\end{figure*}

\end{document}